\documentclass[11pt]{article} 
\usepackage[utf8]{inputenc}
\usepackage{amsmath, amssymb,amsfonts, amsthm, bbm, graphicx, url, hyperref, cleveref}

\hypersetup{colorlinks}
\usepackage{float}
\usepackage{graphicx}
\usepackage{bm}
\usepackage{amssymb}
\usepackage{bbm}
\usepackage{xcolor}
\usepackage{algorithm}
\usepackage{algorithmic}
\usepackage{dsfont}
\usepackage{xspace}
\usepackage{comment}
\usepackage{tikz}
\usepackage{dsfont}
\usepackage{fullpage}

\newcommand{\eat}[1]{}

\newtheorem{result}{Theorem} 

\newtheorem{theorem}{Theorem}[section]
\newtheorem{lemma}[theorem]{Lemma}
\newtheorem{definition}[theorem]{Definition}

\newtheorem{corollary}[theorem]{Corollary}

\newtheorem{claim}[theorem]{Claim}

\newcommand{\X}{\mathcal{X}}

\newcommand{\mC}{\mathcal{C}}
\newcommand{\mD}{\mathcal{D}}
\newcommand{\mU}{\mathcal{U}}
\newcommand{\mH}{\mathcal{H}}

\newcommand{\mA}{\mathcal{A}}

\newcommand{\mL}{\mathcal{L}}
\newcommand{\mLc}{\mathcal{L}_\mathsf{cvx}}

\newcommand{\mP}{\mathcal{P}}

\newcommand{\mLC}{\mathrm{Lin}_\mathcal{C}}

\newcommand{\Bad}{\mathrm{Bad}}

\DeclareMathOperator{\Var}{Var}
\DeclareMathOperator{\Cov}{Cov}

\renewcommand{\Im}{\mathrm{Im}}

\newcommand{\tf}{\tilde{p}}
\newcommand{\tp}{\tilde{p}}
\newcommand{\ty}{\mathbf{\tilde{y}}}

\newcommand{\R}{{\mathbb{R}}}
\newcommand{\Z}{{\mathbb{Z}}}

\newcommand{\y}{\mathbf{y}}
\newcommand{\sy}{\mathbf{y^*}}
\newcommand{\bv}{\mathbf{v}}

\newcommand{\x}{\mathbf{x}}
\newcommand{\z}{\mathbf{z}}

\newcommand{\rgta}{\rightarrow}

\newcommand{\lt}{\left}
\newcommand{\rt}{\right}

\newcommand{\zo}{\ensuremath{\{0,1\}}}
\newcommand{\izo}{\ensuremath{[0,1]}}

\newcommand{\infnorm}[1]{\left\lVert#1\right\rVert_\infty}

\newcommand{\fD}{\mathcal{D}_{\tp}}
\newcommand{\Dv}{{\mathcal{D}|_{\bv}}}

\newcommand{\Pell}{\partial \ell}

\newcommand{\eps}{\varepsilon}

\newcommand{\pmo}{\ensuremath{ \{\pm 1\} }}
\newcommand{\fr}[1]{\ensuremath{\frac{1}{#1}}}

\DeclareMathOperator{\poly}{poly}

\DeclareMathOperator*{\E}{\mathbf{E}}

\DeclareMathOperator*{\argmin}{arg\,min}

\DeclareMathOperator{\Lin}{Lin}
\newcommand{\Ber}{\mathrm{Ber}}

\newcommand{\set}[1] {\left\{ #1 \right\}}

\newcommand{\mce}{\mathsf{MCE}}
\newcommand{\umce}{\mathsf{sMCE}}

\newcommand{\lglm}{\mathcal{L}_\mathsf{GLM}}

 \pdfoutput=1

\title{Swap Agnostic Learning,\\
or Characterizing Omniprediction via Multicalibration\thanks{This manuscript represents the full version of \cite{swap-neurips}. It subsumes arXiv:2302.06726v1 by the authors, under the title \emph{Characterizing notions of omniprediction via multicalibration}.}}
\author{Parikshit Gopalan\\
Apple
\and Michael P. Kim\\
Cornell University\thanks{This work was completed when the author was at the Miller Institute for Basic Research in Science and the Simons Insitute for the Theory of Computing at UC Berkeley.}
\and Omer Reingold\\
Stanford
}
\date{}

\begin{document}

\maketitle

\begin{abstract}

We introduce and study Swap Agnostic Learning.
The problem can be phrased as a game between a \emph{predictor} and an \emph{adversary}:  first, the predictor selects a hypothesis $h$; then, the adversary plays in response, and for each level set of the predictor $\{x \in \mathcal{X} : h(x) = v\}$ selects a (different) loss-minimizing hypothesis $c_v \in \mathcal{C}$; the predictor wins if $h$ competes with the adaptive adversary's loss.
Despite the strength of the adversary, we demonstrate the feasibility Swap Agnostic Learning for any convex loss.

Somewhat surprisingly, the result follows through an investigation into the connections between Omniprediction \cite{omni} and Multicalibration \cite{hkrr2018}.
Omniprediction is a new notion of optimality for predictors that strengthtens classical notions such as agnostic learning.
It asks for loss minimization guarantees (relative to a hypothesis class) that apply not just for a specific loss function, but for any loss belonging to a rich family of losses.
A recent line of work shows that omniprediction is implied by multicalibration and related multi-group fairness notions \cite{omni, lossoi}. This unexpected connection raises the question: {\em is multi-group fairness necessary for omniprediction?} 

Our work gives the first affirmative answer to this question.
We establish an \emph{equivalence} between swap variants of omniprediction and multicalibration and swap agnostic learning.
Further, swap multicalibration is essentially equivalent to the standard notion of multicalibration, so existing learning algorithms can be used to achieve any of the three notions.
Building on this characterization, we paint a complete picture of the relationship between different variants of multi-group fairness, omniprediction, and Outcome Indistinguishability \cite{oi}.
This inquiry reveals a unified notion of OI that captures all existing notions of omniprediction and multicalibration. \end{abstract}

\thispagestyle{empty}
\newpage
\thispagestyle{empty}

{
  \hypersetup{linkcolor=black}
  \tableofcontents
}
\newpage
\setcounter{page}{1}

\section{Introduction}
\label{sec:intro}

Since its inception as an extension to Valiant's PAC framework \cite{valiant1984theory,haussler1992decision,kearns1994introduction}, Agnostic Learning has been the central problem of supervised learning theory.
Agnostic learning frames the task of supervised learning through loss minimization:
given a loss function $\ell$, a hypothesis class $\mC$, and $\eps \ge 0$, a predictor $h$ is an agnostic learner if it achieves loss that competes with the minimal achievable within the hypothesis class.
\begin{gather}
\label{def:AL}
    \E[\ell(\y,h(\x))] \le \min_{c \in \mC}\ \E[\ell(\y,c(\x))] + \eps
\end{gather}
While the agnostic learning paradigm has been remarkably successful, in recent years, researchers have investigated alternative learning paradigms to address concerns of the modern prediction pipeline, including fairness and robustness.
The work of \cite{hkrr2018} introduced \emph{multicalibration} as a new paradigm for learning fair predictors.
Multicalibration asserts fairness as a first-order goal, requiring that predictions appear calibrated even conditioned on membership in one of a potentially-huge collection of subgroups $\mC$ of the domain.\footnote{This collection of subgroups is suggestively denoted by $\mC$, indicating that (in correspondence with a hypothesis class) subgroup membership can be computed using models of bounded capacity, such as small decision trees, halfspaces, or neural networks.}
As a solution concept, multicalibration can be achieved efficiently using a weak learner for $\mC$ to identify subgroups where predictions are miscalibrated.
In contrast to agnostic learning, multicalibration does not make reference to minimizing any loss.

Yet, it turns out that there are deep connections between multicalibration and loss minimization.
This connection was first discovered in the work of \cite{omni}, who introduced the notion of \emph{omniprediction} as a new solution concept in supervised learning.
Intuitively, an omnipredictor is a single predictor that provides an optimality guarantee for many losses \emph{simultaneously}.
More formally, for a collection of loss functions $\mL$ and a hypothesis class $\mC$, a predictor is an $(\mL,\mC)$-omnipredictor if \emph{for any loss in the collection} $\ell \in \mL$, the predictions achieve loss that competes with the minimal achievable within the class $\mC$.
On a technical level, the omniprediction guarantee is made possible by post-processing the predictions after the loss is revealed. 
Importantly though, the post-processing is a data-free optimization that can be performed efficiently, based only on the loss function and prediction, and not the data distribution.
In this way, an omnipredictor can be trained on the dataset once and for all without knowing the specific target loss in advance (beyond the fact that it belongs to $\mL$).

The surprising main result of \cite{omni} demonstrates that multicalibrated predictors are omnipredictors for $\mLc$, the class of all convex loss functions.
In other words, the multicalibration framework is capable of guaranteeing loss minimization in a very strong sense:
learning a single multicalibrated predictor gives an agnostic learner, for every target convex loss $\ell \in \mLc$.
The results of \cite{omni} stand in contrast to the convential wisdom that optimizing predictions for different loss functions requires a separate training procedure for each loss.

On the surface, multicalibration seems like a fundamentally different approach to supervised learning than agnostic learning.
The implication of omniprediction from multicalibration, however, suggests a deeper connection between the notions.
In fact, follow-up work of \cite{lossoi} gave new constructions of omnipredictors (for different loss classes $\mL$), similarly deriving the guarantees from notions of multi-group fairness.
To summarize the state of the art, we have constructions of omnipredictors of various flavors, and all such constructions rely on some variant of multicalibration.

While multicalibration suffices to guarantee omniprediction, a glaring question remains in the development of this theory:
{\em  Is multicalibration necessary for omniprediction?}
A closely-related question stems from the observation that multicalibration guarantees omniprediction, despite the fact that the definition of multicalibration does not reference any loss function to minimize.
It is natural to ask whether we can arrive at the same solution concept through a more standard supervised learning framing:
{\em Can we characterize multicalibration in the language of loss minimization?}
We investigate these questions, exploring the connections between agnostic learning, omniprediction, and multicalibration.

\paragraph{Tight charcteriztions via swap notions.}
To understand the relationship between omniprediction and multicalibraiton, we formulate a new learning task, which we call Swap Agnostic Learning. 
This \emph{swap} variant of agnostic learning is inspired by the notion of swap regret in the online learning literature \cite{FosterV98, BlumM07}, where the learner must achieve vanishing regret, not simply overall, but even conditioned on their decisions. Swap agnostic learning naturally suggests swap notions of omniprediction (and generalizations such as swap loss OI) and multicalibration.
At first glance, swap agnostic learning is a considerably more ambitious goal than standard agnostic learning.
Nevertheless, our work demonstrates an efficient algorithm for swap agnostic learning for any convex loss function, leveraging only a weak agnostic learner for the hypothesis class.

Our main result about swap agnostic learning actually follows as a corollary of a much broader discovery:  \emph{swap agnostic learning and swap variants of omniprediction and multicalibration are all equivalent.}
Concretely, we introduce new swap variants of omniprediction and multicalibration (also inspired by swap regret) and show that each of these definitions---despite different technical framings---actually encodes the same solution concept.
In other words, once we move to swap variants, multicalibration \emph{is} necessary (and sufficient) for omniprediction, and necessary even for the ``weaker'' goal of swap agnostic learning.
Further, we argue that the original formulation of multicalibration is essentially equivalent to our new swap variant; in particular, the original multicalibration algorithm of \cite{hkrr2018} actually produces a swap multicalibrated predictor.
Thus, all three swap variants can be achieved efficiently given a weak agnostic learner for the class $\mC$.

While the swap and standard variants of multicalibration are quite similar, for agnostic learning and omniprediction, the distinction is substantial.
We prove separations between standard and swap learning for these notions.
In combination, our results shed light on the question of whether (standard) omniprediction implies multicalibration, suggesting that the answer is no:  (swap) multicalibration is equivalent to the stronger notion of swap omniprediction, which is provably stronger than standard omniprediction.

Our results provide an exact characterization of these swap learning notions, as well as relationships between other learning desiderata explored in recent works.
At the extreme, we introduce a swap variant of Loss Outcome Indistinguishability \cite{oi,lossoi} and show how it is expressive enough to capture all existing notions in the study of multicalibration and omniprediction.
We present a tight chrarectization of it in terms of swap multicalibration for a suitably defined class $\mC'$ which depends on both the loss family $\mL$ and the hypothesis class $\mC$. Standard Loss OI (without the swap) was known to be implied by calibrated multiaccuracy for the same class $\mC'$.
Our result shows that strenghting to swap loss OI, gives a tight characterizaiton in terms of swap multicalibration.\footnote{At this point, we have introduced multiplce notions of swap learning and omniprediction. For the curios reader, the precise relationship between the different notions is depicted in Figure~\ref{fig:rel}.}
We describe each of the notions, and our contributions, in detail in the next section.

\section{Our Contributions:  Defining and Characterizing Swap Learning}
\label{sec:defs}

In this work, we give new definitions of swap learning notions.
A major contribution of this work is defining these notions precisely, within the context of prior work on supervised learning and algorithmic fairness.
After defining these notions, we give a tight characterization of them, showing key equivalences and key separations.
In this section, we give a technical overview of the definitions and the results.

We begin with the definitions we introduce to study swap notions of learning.
We recall prior notions---agnostic learning, omniprediction, multicalibration, and loss OI---and we present our new swap variants of each.
Along the way, we include a number of claims about the relationship between the standard and swap variants.
Then, with the definitions in place, we give an overview of our results characterizing the notions.

\paragraph{Preliminaries.}
We work in the {\em agnostic} learning setting, where we assume a data distribution $(\x,\y) \sim \mD$ supported on $\X \times \zo$.
We take the objects of study in agnostic learning to be real-valued hypotheses $h:\X \to \R$.
Omniprediction and multicalibration study the special case of predictors $\tp:\X \to [0,1]$ that map domain elements to probabilities.
We denote the Bayes optimal predictor as $p^*(x) = \Pr[\y = 1|\x = x]$.
As is standard in agnostic learning, we make no assumptions about the complexity of $p^*:\X \rgta [0,1]$.
While, in principle, predictors and hypotheses may take on continuous values in $\R$, we restrict our attention to functions supported on finitely-many values.
We let $\Im(\tf) \subseteq [0,1]$ denote the set of values taken on by $\tf(x)$.

Given a hypothesis $\tp:\X \to \R$, it will be useful to imagine drawing samples from $\mD$ in two steps:  first, we sample a prediction $\bv \in \R$ according to the distribution of $\tp(\x)$ under $\mD$; then, we draw the random variables $(\x,\y)$ according to the conditional distribution $\mD \vert \tp(\x) = \bv$.
Accordingly, let $\mD_{\tp}$ denote the distribution of $\tp(\x)$ where $\x \sim \mD$.
When the predictor $\tp$ is clear from context, we use the following notation:  for each $v \in \Im(\tf)$, let $\mD|_v$ denote the conditional distribution $\mD|\tp(\x) =v$.
Note that the distribution $\Dv$ for randomly drawn $\bv \sim \fD$ recovers the data distribution $\mD$.

Throughout, we consider loss functions $\ell: \zo \times \R \rgta \R$, which take a binary label $y$ and a real-valued action $t$ and assign to it a loss value.
Two common examples are the squared loss $\ell_2(y, t) = (y -t)^2$ and logistic loss $\ell_{\mathrm{lg}}(y,t) = \log(1 + \exp((1-2y)t)$.
Of key interest is the loss class $\mLc$, which denotes the set of all convex losses satisyfing a set of ``niceness'' conditions, defined formally in Definition~\ref{def:nice-loss}.

\subsection{Definitions of Swap Learning and Swap Multicalibration}

\subsubsection*{Swap Agnostic Learning}

Swap agnostic learning is defined with respect to a loss function $\ell:\zo \times \R \to \R$ and a hypothesis class $\mC \subseteq \{c:\X \to \R\}$, and can be viewed as a game with two participants:
\begin{itemize}
\item {\bf The Predictor $\mP$} selects a hypothesis $h:\X \to \R$, aiming to minimize the loss $\ell$.
\item {\bf The Adversary $\mA$} plays in response:  for each level set $\{x \in \X : h(x) = v\}$, the adversary may choose a separate loss-minimizing hypothesis $c_v \in \mC$.
\end{itemize}

Swap agnostic learning requires that the predictions of $h$ are competitive with the adaptive strategy determined by the adversary $\set{c_v : v \in \Im(h)}$, even though the predictor plays first.
\begin{definition}[Swap Agnostic Learning]
    For a loss function $\ell$, hypothesis class $\mC$, and error $\eps \ge 0$, a hypothesis $h$ is a $(\ell,\mC,\eps)$-swap agnostic learner if
    \begin{gather}
    \label{def:eqn:SwapAL}
        \E[\ell(\y,h(\x))] \le \E_{\bv \sim \mD_h}\left[\min_{c_v \in \mC}~ \E[\ell(\y,c_{v}(\x) \ \vert \ h(\x) = \bv] \right]  + \eps.
    \end{gather}
\end{definition}
We borrow nomenclature from online learning: a predictor using a swap agnostic learner $h$ has no incentive to ``swap'' any of their fixed predictions $h(\x) = v$ to predict according to $c_v \in \mC$.
Contrasting the requirement in (\ref{def:eqn:SwapAL}) to that of agnostic learning in (\ref{def:AL}), we have switched the order of quantifiers, such that the minimization is taken after the expectation over the choice of $h(\x) = \bv$.

Swap agnostic learning strengthens standard agnostic learning, where the predictor only competes against the single best hypothesis $c \in \mC$.
\begin{claim}
If $\tp$ is a $(\ell,\mC,\eps)$-swap agnostic learner, it is a $(\ell,\mC,\eps)$-agnostic learner.
\end{claim}
Indeed, swap agnostic learning seems to be a much more stringent condition.
Provided the class $\mC$ contains all constant functions, the adversary can simply imitate the predictor when $h(x) = v$ by choosing $c_v(x) = v$.
For such $\mC$, (\ref{def:eqn:SwapAL}) implies that imitation is the best strategy for the adversary (up to $\eps$).

\paragraph{Notation.}
In this presentation, we use the notation $c_v \in \mC$ to emphasize the fact that the minimization occurs after conditioning on the predicted value $\bv = v$.
Of course, the choice of $c_v$ is still a search over $\mC$, and thus, equivalently could be written as a minimization over $c \in \mC$.
We use these notations interchangeably.

\subsubsection*{Swap Omniprediction}
We recall the original notion of omniprediction proposed by \cite{omni}, which requires a predictor to yield an agnostic learner simultaneously for every loss in a collection $\mL$.
As in \cite{omni, lossoi}, we observe that for any loss function $\ell$ and known distribution on outcomes $\y \sim \Ber(p)$ for $p \in [0,1]$, there is an optimal action $k_\ell(p)$ that minimizes the expected loss $\ell$.
Formally, we define the optimal post-processing of predictions by the function $k_\ell:[0,1] \to \R$, which specifies the action that minimizes the expected loss.\footnote{If there are multiple optima, we break ties arbitrarily.}
\begin{align}
\label{eq:k-ell}
    k_\ell(p) = \argmin_{t \in \R} \E_{\y \sim \Ber(p)}[\ell(\y, t)].
\end{align}
This observation is particularly powerful because the function $k_\ell$ is given by a simple, univariate optimization that can be used as a data-free post-processing procedure on top of a predictor $\tp$.
Concretely, it allows us to define  omniprediction, which requires that the function $h_\ell = k_\ell \circ \tp$ is an $(\ell, \mC, \eps)$-agnostic learner for every $\ell \in \mL$. 
 
\begin{definition}[Omnipredictor, \cite{omni}]
    For a collection of loss functions $\mL$, a hypothesis class $\mC$, and error $\eps > 0$,
    a predictor $\tp: \X \rgta \izo$ is an $(\mL, \mC, \eps)$-omnipredictor if for every $\ell \in \mL$, 
    \begin{gather}
    \E[\ell(\y, k_\ell(\tp(\x))] \le \min_{c \in \mC}\ \E[\ell(\y, c(\x))] + \eps.
    \end{gather}
\end{definition}
We propose a strengthened notion of Swap Omniprediction, where the adversary may choose both the loss $\ell_v \in \mL$ and hypothesis $c_v \in \mC$ based on the prediction $\tp(x) = v$.
For each supported $v \in \Im(\tf)$, we allow an adversarial choice of loss $\ell_v \in \mL$; we call the collection $\set{\ell_v \in \mL : v \in \Im(\tf)}$ an assignment of loss functions.
For each $v$, se compare the omnipredictor's action $k_{\ell_v}(v)$ with the optimal choice of hypothesis $c_v$ tailored to the assigned loss $\ell_v$ over the conditional distribution $\mD|\tp(\x) = v$. 
\begin{definition}[Swap Omnipredictor]
\label{def:sr-omni}
For a collection of loss functions $\mL$, a hypothesis class $\mC$, and error $\eps > 0$,
a predictor $\tp: \X \rgta \izo$ is a $(\mL,\mC,\eps)$-swap omnipredictor if for any assignment of loss functions $\{\ell_v \in \mL\}_{v \in \Im(\tp)}$,
\begin{gather}
    \label{def:eqn:swap:omni}
    \E_{\bv \sim \mD_{\tp}}\Big[ \E[ \ell_{\bv}(\y,k_{\ell_{\bv}}(\bv)) \ \vert \ \tp(\x) = \bv]\Big] \le
    \E_{\bv \sim \mD_{\tp}}\left[\min_{c \in \mC}\ \E[\ell_{\bv}(\y, c(\x)) \ \vert \ \tp(\x) = \bv]\right] + \eps.
\end{gather}
\end{definition}
Swap omniprediction gives the adversary considerable power.
For instance, the special case where we restrict the adversary's choice of losses to be constant, $\ell_v = \ell$, realizes swap agnostic learning for $\ell$.
\begin{claim}
\label{claim:swap-strong}
    If $\tf$ is a $(\mL, \mC, \eps)$-swap omnipredictor, it is a $(\ell, \mC, \eps)$-swap agnostic learner for every $\ell \in \mL$, and hence a $(\mL, \mC, \eps)$-omnipredictor.
\end{claim}
Analogous to the standard notions (where omniprediction implies agnostic learning), swap omniprediction implies swap agnostic learning for every $\ell$; however, swap omniprediction gives an even stronger guarantee since the adversary's loss $\ell_v$ may be chosen in response to the prediction $\tp(\x) = v$.

\subsubsection*{Swap Multicalibration}

Multicalibration was introduced in the work of \cite{hkrr2018} as a notion of algorithmic fairness following \cite{KleinbergMR17}, but has since seen broad application across many learning applications (e.g., \cite{oi,omni,kim2022universal,GuptaJNPR22}).
Informally, multicalibration requires predictions to appear calibrated, not simply overall, but also when we restrict our attention to subgroups within some broad collection $\mC$.
The formulation below appears in \cite{lossoi}. 

\begin{definition}[Multicalibration, \cite{hkrr2018}]
\label{def:s-mc}
For a hypothesis class $\mC$ and $\alpha \ge 0$, 
a predictor $\tp:\X \rgta [0,1]$ is $(\mC, \alpha)$-multicalibrated if
\begin{align}
    \max_{c \in \mC} \E_{\bv \sim \fD} \lt[\ \lt| \E_{\phantom{\mD}} [c(\x)(\y - \bv) \ \vert \ \tp(\x) = \bv] \rt|\ \rt] \leq \alpha. 
\end{align}
\end{definition}
When $c: \X \to \zo$ is Boolean (and has sufficiently large measure), this definition says that conditioned on $c(\x) = 1$, the calibration violation is small, recovering the definition in \cite{hkrr2018}.

Swap Multicalibration strengthens multicalibration, extending the pseudorandomness perspective on multicalibration developed in \cite{oi, omni}.
Swap multicalibration requires that for the typical prediction $\bv \sim \fD$, no hypothesis in $c_v \in \mC$ achieves good correlation with the residual labels $\y - \bv$ over $\mD$ conditioned on $\tp(\x) = \bv$.

\begin{definition}[Swap Multicalibration]
\label{def:u-mc}
For a hypothesis class $\mC$ and $\alpha \ge 0$, 
a predictor $\tp:\X \rgta [0,1]$ is $(\mC, \alpha)$-swap multicalibrated if
\begin{align}
        \E_{\bv \sim \fD} \lt[\ \max_{c \in \mC}\ \lt|\E_{\phantom{\mD}} [c(\x)(\y - \bv) \ \vert \ \tp(\x) = \bv]  \rt|\ \rt] \leq \alpha. 
    \end{align}
\end{definition}

Again, the difference between swap and standard multicalibration is in the order of quantifiers.
The standard definition requires that for every $c \in \mC$, the correlation $|\E_{\Dv}[c(\x)(\y - \bv)]|$ is small in expectation over $\bv\sim \fD$.
Swap multicalibration considers the maximum of $|\E_{\Dv}[c(\x)(\y - \bv)]|$ over all $c \in \mC$ for each fixing of $\bv =v$ and requires this to be small in expectation over $\bv \sim \fD$. 
It follows that swap multicalibration implies standard multicalibration.
\begin{claim}
\label{claim:swap-mc}
    If $\tp$ is $(\mC,\alpha)$-swap multicalibrated, it is $(\mC,\alpha)$-multicalibrated.
\end{claim}

Swap multicalibration allows an auditor to use a different choice of $c_\bv \in \mC$ to audit our predictor, conditioned on each value $\bv$ of the predictions.
In this sense, it is reminiscent of our definitions of swap loss minimization.
But while swap multicalibration is nominally a stronger notion that than standard multicalibration, the complexity of achieving swap multicalibration is essentially the same as multicalibration.
We present two arguments to support this view. 
\begin{itemize}
\item All known algorithms for achieving standard multicalibration \cite{hkrr2018, omni} already give swap multicalibration. They ensure multicalibration by conditioning on each prediction value, and then using a weak agnostic learner for $\mC$ to audit for a violation of the multicalibration condition. This approach naturally yields swap multicalibration. Thus, the same primitive of weak agnostic learning \cite{SBD2} needed for standard multicalibration (and standard agnostic learning) also suffices for swap multicalibration.

\item We show that starting from any $(\mC, \alpha)$-multicalibrated predictor $\tf$, by suitably discretizing its values, we obtain a predictor that is close to $\tf$ and $(\mC, \alpha'$)-swap multicalibrated, with some degradation in the value of $\alpha'$.

\end{itemize}

We formalize both these arguments in Section~\ref{sec:properties}.

\subsubsection*{Swap Loss Outcome Indistinguishability}

The connections between multi-group fairness and omniprediction were extended in \cite{lossoi}.
Building on the original Outcome Indistinguishability framework of \cite{oi}, \cite{lossoi} introduced a related notion they referred to as Loss Outcome Indistinguishability (Loss OI).
Intuitively, Loss OI requires that outcomes sampled from the predictive model $\tp$ are indistinguishable from Nature's outcomes, according to tests defined by the loss class $\mL$ and hypothesis class $\mC$.
They showed that $(\mL, \mC)$-Loss OI implies the standard $(\mL,\mC)$-omniprediction from \cite{omni}, and that it is implied by a multi-group fairness notion called calibrated multiaccuracy (for a class of hypotheses $\mC'$ that may depend on both $\mL$ and $\mC$).
For a fixed hypothesis class, calibrated multiaccuracy implies multiaccuracy but is weaker than multicalibration.
The Loss OI framework can be used to derive a wide range of omniprediction guarantees starting from multigroup fairness notions of varying strength.
In this work, we extend the framework by introducing a swap variant.

Recalling the definition put forth by \cite{oi,lossoi}, Outcome Indistinguishability requires predictors $\tf$ to fool a family $\mU$ of statistical tests $u: \X \times \izo \times \zo$ that take a point $\x \in \X$, a prediction $\tf(\x) \in \izo$ and a label $\y \in \zo$ as their arguments.
Formally, we use $(\x,\y^*)$ to denote a sample from the true joint distribution over $\X \times \zo$.
Then, given a predictor $\tp$, we associate it  with the random variable with $\E[\ty|x] = \tp(x)$, i.e., where $\ty|x \sim \Ber(\tp(x))$.
The variable $\ty$ can be viewed as $\tp$'s simulation of Nature's label $\sy$.\footnote{In this presentation, we abuse notation and let $\mD$ to denote the joint distribution $(\x, \sy, \ty)$, where $\E[\y^*|x] = p^*(x)$ and $\E[\ty|x] = \tp(x)$.
While the joint distribution of $(\y^*, \ty)$ is not important to us, for simplicity we assume they are independent given $\x =x$.}
The goal is distinguish between the scenarios where $\y = \y^*$ is generated by {\em nature} versus where $\y =\ty$ is a simulation of nature according to the predictor $\tf$.

Formally, we require than for every $u \in \mU$, 
\[ \E_{\mD} [u(\x, \tf(\x), \sy)] \approx_\eps \E_{\mD} [u(\x, \tf(\x), \ty)]. \] 
Loss OI specializes this to a specific family of tests arising in the analysis of omnipredictors.
\begin{definition}[Loss OI, \cite{lossoi}]
\label{def:loss-oi}
    For a collection of loss functions $\mL$, hypothesis class $\mC$, and $\eps \ge 0$, define the family of tests $\mU(\mL, \mC) = \{u_{\ell,c}\}_{\ell \in \mL, c \in \mC}$ where
    \begin{align} 
    \label{eq:ulc}
        u_{\ell, c}(x, v , y) = \ell(y, k_\ell(v)) - \ell(y, c(x)). 
    \end{align}
    A predictor $\tf:\X \to \izo$ is $(\mL, \mC, \eps)$-loss OI if for every $u \in \mU(\mL, \mC)$, it holds that
    \begin{align} 
    \label{eq:loss-oi} 
        \lt| \E_{(\x, \y^*) \sim \mD} [u(\x, \tf(\x), \y^*)] - \E_{(\x, \ty) \sim \mD} [u(\x, \tf(\x), \ty)] \rt| \leq \eps.
    \end{align}
\end{definition}
By design, Loss OI implies omniprediction.
\begin{claim}[Proposition 4.5, \cite{lossoi}]
    If the predictor $\tf$ is $(\mL, \mC, \eps)$-loss OI, then it is an $(\mL, \mC, \eps)$-omnipredictor.
\end{claim}
Indeed, if the expected value of $u$ is nonpositive for all $u \in \mU(\mL,\mC)$, then $\tp$ must achieve loss competitive with all $c \in \mC$.
The argument leverages the fact that $u$ must be nonpositive when $\ty \sim \Ber(\tp(\x))$---after all, in this world $\tp$ is the Bayes optimal.
By indistinguishability, $\tp$ must also be optimal in the world where outcomes are drawn as $\sy$.
The converse, however, is not always true.

Here, we introduce Swap Loss OI.
In this swap variant, we allow the choice of distinguisher to depend on the predicted value.
\begin{definition}[Swap Loss OI]
For a collection of loss functions $\mL$, hypothesis class $\mC$ and $\eps \ge 0$, for an assignment of loss functions $\{\ell_v \in \mL\}_{v \in \Im(\tf)}$ and hypotheses $\{ h_v \in \mH\}_{v \in \Im(\tf)}$, denote $u_v = u_{\ell_v, c_v} \in \mU(\mL, \mC)$.
A predictor $\tf$ is $(\mL, \mC, \alpha)$-swap loss OI if for all such assignments,
\begin{align*}
    \E_{\bv \sim \fD} \lt| \E_{\Dv}[u_\bv(\x, \bv, \y^*)- u_\bv(\x, \bv, \ty)] \rt| \leq \alpha.
\end{align*}
\end{definition}
Swap Loss OI naturally strenthens the standard variant of Loss OI:
in Loss OI, $\ell, c$ are fixed and cannot depend on $\bv$.
In fact, Swap Loss OI also strengthens swap omniprediction.
\begin{lemma}
\label{lem:swap-strong2}
If the predictor $\tf$ satisfies $(\mL, \mC, \alpha)$-swap loss OI, then
\begin{itemize}
\item it is $(\mL, \mC, \alpha)$-loss OI.
\item it is an $(\mL, \mC, \alpha)$-swap omnipredictor.
\end{itemize}
\end{lemma}
We prove this lemma in Section~\ref{sec:lossOI}.

\subsection{Characterizing Swap Learning}
\label{sec:overview}

With the definitions of swap learning notions in place, we can state our main results.

\paragraph{Equivalence of swap notions.}
Our first main result (Theorem \ref{thm:main}) shows the following equivalence between swap notions of multicalibration, omniprediction and agnostic learning.
Specifically, the theorem shows that multicalibration is intimately related to omniprediction for the set of convex loss functions $\mLc$.\footnote{For techincal reasons, the loss class $\mLc$ consists of all convex loss functions that also satisfy a few ``niceness'' properties, stated formally in Definition~\ref{def:nice-loss}).}
\begin{result}[Informal statement of Theorem \ref{thm:main}]
\label{result:main}
 For any hypothesis class $\mC$ and predictor $\tp: \X \to \zo$, the following notions are equivalent:
 \begin{itemize}
 \item $\tp$ is $\mC$-swap multicalibrated.
 \item $\tp$ is an $(\mLc, \mC)$-swap omnipredictor.
 \item $\tp$ is an $(\ell_2, \mC)$-swap agnostic learner.
 \end{itemize}
\end{result}
In other words, once we move to the swap variants of these notions, minimizing the squared error ($\ell_2$) suffices to minimize \emph{every} convex loss function.
By the equivalence to multicalibration, we immediately obtain an efficient algorithm for achieving swap agnostic learning for any convex loss, using only an oracle for standard agnostic learning (see Section~\ref{sec:properties}).

\paragraph{Swap Loss OI is multicalibration over augmented class.}
Our second main result (Theorem~\ref{thm:swap-eq}) gives a characterization of Swap Loss OI in terms of swap multicalibration.
As we discussed, swap loss OI is the strongest notion of omniprediction to date: all existing notions of omniprediction can be implemented via swap loss OI.
We show that this indistinguishability notion can actually be defined as a variant of multicalibration.
In order to state the result, we recall the notion of discrete derivatives of loss functions used in \cite{lossoi}.
\begin{definition}
    For a loss function $\ell: \zo \times \R \to \R$, define $\partial \ell: \R \to \R$ as
    \[ \Pell (t) = \ell(1, t) - \ell(0, t).\]
    For a family of loss functions $\mL$ and  a hypothesis class $\mC$, let the classes of discrete derivatives be given as $\partial \mL = \{\Pell\}_{\ell \in \mL}$ and $\partial \mL \circ \mC = \{ \Pell \circ c\}_{\ell \in \mL, c\in \mC}$.  
\end{definition}
We prove the equivalence of $(\mL, \mC)$-swap loss OI (for loss classes $\mL$ that satisfy the same niceness property as above and contain the squared loss) and $\partial \mL \circ \mC$-swap multicalibration.

\begin{result}[Informal statement of Theorem \ref{thm:swap-eq}]
Let $\mL$ be a family of ``nice'' loss functions containing the squared loss $\ell_2$.  For any hypothesis class $\mC$, the following notions are equivalent:
\begin{itemize}
    \item $(\mL, \mC)$-swap loss OI.
    \item $(\partial \mL \circ \mC)$-swap multicalibration.
\end{itemize}
\end{result}
Interestingly, the same class $\partial L \circ \mC$ appears in the work of \cite{lossoi}, where they show that $(\mL, \mC)$-loss OI is implied by calibration and $\partial L \circ \mC$-multiaccuracy.
We show that the stronger notion of $\partial L \circ \mC$-multicalibration in fact gives an equivalence to $(\mL, \mC)$-swap loss OI.
Along the way, we prove that swap loss OI for just the squared loss implies calibration.

\begin{figure}[t]
    \centering
    \includegraphics[width=0.75\textwidth]{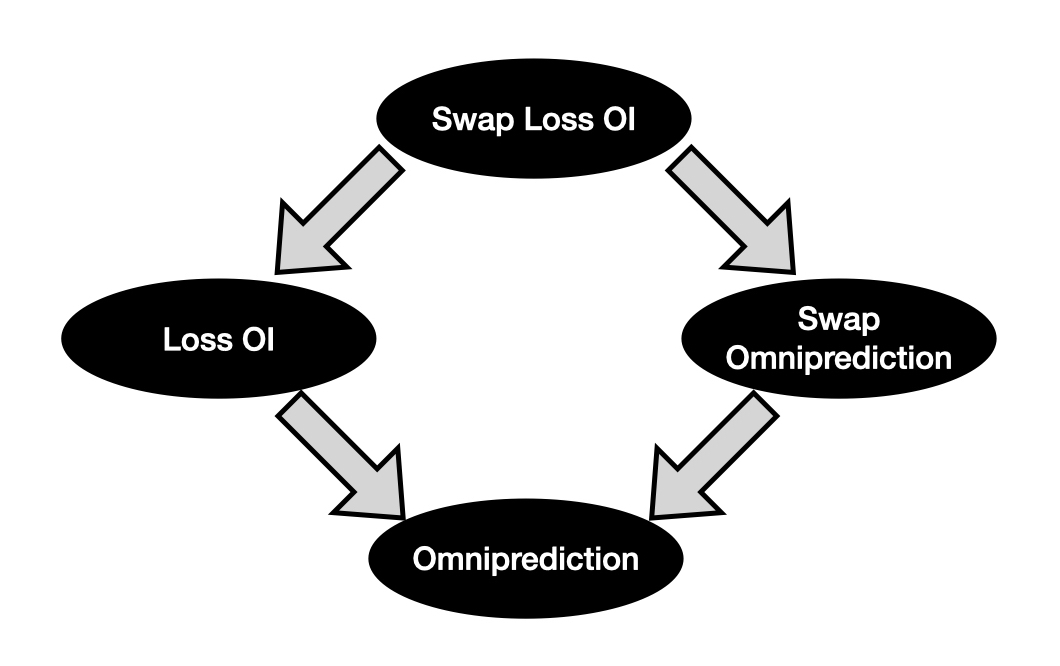}
    \caption{Relation between notions of omniprediction}
     \label{fig:rel}
\end{figure}

\paragraph{Separating notions of omniprediction.}
Rounding out our main results, we paint a complete picture of the relationships between various omniprediction notions by establishing implications and separations between them.
The implications are illustrated in Figure \ref{fig:rel}.
That loss OI implies omniprediction (without swap) was proved in \cite{lossoi}; the remaining implications are proved in Section \ref{sec:example}.
We complement these implications with the following separation result.
\begin{result} 
$(\mL, \mC)$-Loss OI and $(\mL, \mC)$-swap omniprediction are incomparable.
\end{result}
The separation result is proved using the $\ell_2$ and $\ell_4$ loss, both of which are convex.
The theorem statement follows as a corollary of the lemmas proved in Section~\ref{sec:example}.
Our results show that Figure~\ref{fig:rel} cannot be simplified further for generic choices of $\mL$ and $\mC$ (although other implications could hold for specific loss and hypothesis classes).
In all, we conclude the following statements:
\begin{itemize}
\item Swap loss OI is more powerful than loss OI, and swap omniprediction is more powerful than standard omniprediction. 
\item Swap loss OI is the most powerful of the four notions, whereas standard omniprediction is the least powerful. 
\end{itemize}

\paragraph{Summary of contributions.}
Our work adds a new and significant connection in the theory of agnostic learning and multi-group fairness.
In particular, by borrowing the notion of swapping from online learning, we uncover surprisingly-powerful, but feasible solution concepts for supervised learning.
Namely, we can achieve the swap agnostic learning notions assuming only sample access to the data source and a weak agnostic learner for the collection $\mC$.
The equivalence we draw between swap agnostic learning for the squared error and multicalibration with swap omniprediction \emph{for all convex losses} highlights the power of simple goals like squared error minimization and calibration.

Our work gives a complete characterization of the notions we study at the ``upper end'' of strength (i.e., swap variants).
A fascinating outstanding question to address in future research is whether there is a similar characterization of \emph{standard} omniprediction in terms of multi-group fairness notions.

\subsection{Discussion and Related Work}
\label{sec:intro-sep}

Theorem~\ref{result:main} shows the first {\em reverse} connection between an omniprediction notion and multi-group fairness.
It tells us that multicalibration, while it does not explicitly mention any loss function, actually \emph{requires} loss minimization in a deep sense. 
Note that convexity plays a key role in the statement; indeed, an example from \cite{omni} shows that $\mC$-swap multicalibration does not give (even standard) omniprediction for all non-convex losses.

The equivalence to swap agnostic $\ell_2$-loss minimization shows that the problem of finding a multicalibrated predictor can be formulated as a loss minimization problem, using the language of swap agnostic learning. 
One can also show  that $(\ell_2, \mC)$-agnostic learning (without the swap) results in a predictor that guarantees $\mC$-multiaccuracy.
A similar result appears in \cite[Theorem 5.6]{lossoi}.
In a sense, our results characterizing multicalibration can be viewed as a lifting of the earlier result on multiaccuracy to a more powerful (swap) model.
This ``hierarchy'' of implications bears similarity to results from the outcome indistinguishability literature, where multicalibration corresponds to indistinguishability in the prediction-access model, whereas multiaccuracy corresponds to indistinguishability in the  weaker no-access model \cite{oi}.

\paragraph{Multicalibration as no-swap-regret relative to $\mC$.}

Our notions of swap loss minimization and swap omniprediction are inspired by notions of swap regret minimization in online learning \cite{FosterV98, FV99, BlumM07}. These classic results of \cite{FosterV98, FV99} relate calibration and swap regret, whereas \cite{BlumM07} show a generic way to convert a low external regret algorithm to ones with low internal and swap regret. The study of online learning and regret minimization is extensive, with deep connections to calibration and equilibria in game theory, we refer the readers to \cite{CBL, BMchapter} for comprehensive surveys.

Concretely in our language, Foster and Vohra \cite{FosterV98} characterize calibration as $(\ell_2, V)$-swap agnostic learning, where $V = \{v \in [0,1]\}$ is the set of all constant predictors.
Since they work in the online setting, they phrase their result in terms of swap regret relative to strategies in $V$ (the terminology of swap regret comes from \cite{BlumM07}). Since we work in the batch setting, we can rephrase their result in terms of swap agnostic learning, which is a batch analogue of swap regret.  \cite[Theorem 3]{FosterV98} characterizes the calibration score of a predictor $\tp$ in terms of the regret it suffers relative to all constant predictions in $V$ measured under the squared loss (which measures how much it gains by predicting $w \in [0,1]$ in place of $v \in \Im(\tp)$).
Theorem \ref{result:main} gives a similar characterization for $\mC$-multicalibrated predictors, where the no-regret guarantee now applies to strategies $c \in \mC$.
A $\mC$-multicalibrated predictor cannot reduce its squared loss by predicting $c \in \mC$ in place of $v \in \Im(\tp)$, moreover this connection goes both ways.

Recent work has established rich connections between online learning and regret minimization on one hand, and multigroup fairness notions on the other: 
multicalibration in the online setting has been considered in the work of \cite{GuptaJNPR22}, while \cite{BlumLyk20, GillenJKR18} relate multigroup fairness to questions in online learning. Our work contributes a new facet to these connections: it introduces new notions motivated by the regret minimization literature to agnostic loss minimization and omniprediction, and shows that these new notions have tight characterizations in terms of multi-group fairness notions.

\paragraph{Comparing notions for $\mLc$.}

Convex losses play a key role in many of the proofs and equivalences that we establish.
In particular, for each notion of omniprediction / OI we consider, instantiated with $\mL = \mLc$ (for arbitrary $\mC$), there is a corresponding multi-group fairness notion that guarantees the notion.
We list the simplest known multi-group guarantee for achieving each notion in Table \ref{tab:compare}.
Observe that the implications for the swap notions go both ways---the multi-group guarantee actually characterizes the notion of loss minimization.
For the non-swap versions of omniprediction and loss-OI, a tight characterization in terms of some multigroup notion is unknown.
Concretely, we leave open the question of whether there is a multi-group characterization of $(\mLc,\mC)$-omniprediction, which has been alluded to since the original omniprediction work \cite{omni,lossoi}.

\begin{table}
\centering
\begin{tabular}{||l | l |l ||} 
\hline
 Loss minimization notion & Multi-group fairness guarantee & Reference\\
 \hline
 \hline 
 $(\mLc, \mC)$-swap loss OI & $\partial \mLc \circ \mC$-swap multicalibration & Theorem \ref{thm:swap-eq} ($\Leftrightarrow$)\\ [0.5ex]
 $(\mLc, \mC)$-swap omniprediction & $\mC$-swap multicalibration  & Theorem \ref{thm:main} ($\Leftrightarrow$)\\ [0.5ex]
 $(\mLc, \mC)$-loss OI & Calibration + $\partial \mLc \circ \mC$- multiaccuracy  & \cite{lossoi} ($\Leftarrow$)\\ [0.5ex]
 $(\mLc, \mC)$-Omniprediction & $\mC$- multicalibration & \cite{omni} ($\Leftarrow$)\\ [0.5ex]
 \hline
\end{tabular}
\caption{Computational guarantees needed for $(\mLc, \mC)$-loss minimization}
\label{tab:compare}
\end{table}

 \label{sec:related}

\paragraph{Multi-group fairness.}
Notions of multi-group fairness were introduced in the work of \cite{hkrr2018,kgz} and \cite{KNRW17}, following \cite{KleinbergMR17}. The flagship notion of multicalibration has been extended to several other settings including multiclass predictions \cite{dwork2022beyond}, real-valued labels \cite{GuptaJNPR22, JungLPRV21}, online learning \cite{GuptaJNPR22} and importance weights \cite{gopalan2021multicalibrated}.  Alternate definitions and extensions of the standard notions of multicalibration have been proposed in \cite{omni, GopalanKSZ22, DengDZ23, DworkLLT}. Multicalibration has also proved to have unexpected connections to many other domains, including computational indistinguishability \cite{oi}, domain adaptation \cite{kim2022universal}, and boosting \cite{omni, GHHKRS}.

\paragraph{Independent concurrent work.}

Independent work of Dwork, Lee, Lin and Tankala \cite{DworkLLT} defines the notion of \emph{strict multicalibration}.
They connect this notion to the Szemeredi regularity lemma and its weaker version due to Frieze and Kannan. Their notions bears important similarities to swap multicalibration, but it is different.
Like swap multicalibration,  strict multicalibration involves switching the order of expectations and max.
But they require a statistical closeness guarantee, whereas we only require a {\em first order guarantee}, that $c(\x)$ be uncorrelated with $\y - \bv$ conditioned on $\bv$.
For Boolean $c$, these notions are essentially equivalent.
In the setting of real-valued functions $\mC$, however, they seem rather different. Swap multicalibration is closed under linear/convex combinations of $\mC$, whereas strict multicalibration might not be. In the real-valued setting,  strict multicalibration seems to require a stronger learning primitive than weak agnostic learner for $\mC$, which suffices for swap multicalibration. The analogous statistical distance notion of multiaccuracy is at least as powerful as the notion of hypothesis OI for arbitrary loss families $\mL$, which is known to require weak learning for $\partial \mL \circ \mC$ \cite[Theorem 4.9]{lossoi}.

The independent work of Globus-Harris, Harrison, Kearns, Roth and Sorrell \cite{GHHKRS} relates multicalibration to real-valued boosting to minimize $\ell_2$ loss.
The implication that $(\ell_2, \mC)$ swap-agnostic learning implies multicalibration follows from Part(1) of Theorem 3.2 in their paper.
They prove that a violation of multicalibration leads to a better strategy for the adversary in the $(\ell_2, \mC)$-swap minimization game. They do not consider the notion of swap multicalibration, so their result is not a tight characterization unlike our equivalence of Items (1) and (3) in Theorem \ref{thm:main}.

\subsection*{Organization of Manuscript}
The remainder of the manuscript is structured as follows.
In Section~\ref{sec:properties}, we give a deeper overview of swap multicalibration, arguing that it should be thought of as essentially equivalent to standard multicalibration.
In Section~\ref{sec:main}, we prove the main equivalence of the three notions of swap learning.
In Section~\ref{sec:lossOI}, we give the characterization of swap loss OI in terms of swap mulitcalibration over the augmented class.
Finally, in Section~\ref{sec:example}, we prove separation results, establishing the remaining non-implications depicted in Figure~\ref{fig:rel}.

\section{Properties of Swap Multicalibration}
\label{sec:properties}

In this section, we formalize the intuitive claim that standard multicalibration as defined by \cite{hkrr2018} is essentially equivalent to swap multicalibration.
We start, however, with a concrete difference between the standard definition of multicalibration and the swap variant.
Starting from the definitions, we denote the multicalibration error of $\tf$ with respect to $\mC$ under $\mD$ as
\[  \mce_\mD(f, \mC) = \max_{c \in \mC} \E_{\bv \sim \fD} \lt[ \lt| \E_{\Dv} [c(\x)(\y - \bv)] \rt| \rt] \]
and the swap multicalibration error of $\tf$ with respect to $\mC$ under $\mD$ as
\[  \umce_\mD(\tf, \mC) = \E_{\bv \sim \fD} \lt[ \max_{c \in \mC} \lt|\E_{\Dv} [c(\x)(\y - \bv)]  \rt| \rt]. \]

To contrast the notions, consider the following notions of ``bad'' prediction intervals for a given $\mC$ and predictor $\tf$.
\begin{definition}
    Fix a hypothesis class $\mC$.
    Let $\beta \in \izo$.
    For each $c \in \mC$, let the $\beta$-bad prediction intervals for $c$ be given as the $v \in \izo$ with multicalibration violation of at least $\beta$.
    \begin{align*} 
        \Bad_{c,\beta}(\tf) = \lt\{v \in [0,1]: \lt|\E_{\mD|_v}[c(\x)(\y - v)]\rt| \geq \beta \rt\}
    \end{align*}
    Further, let the global $\beta$-bad prediction intervals for $\mC$ be the union over $c \in \mC$.
    \begin{align*}
        \Bad_{\mC,\beta}(\tf) = \bigcup_{c \in \mC} \Bad_{c,\beta}(\tf) = \lt\{v \in [0,1]: \exists c \in \mC\  \mathrm{s.t.}\  \lt|\E_{\mD|_v}[c(\x)(\y - v)]\rt| \geq \beta \rt\}. 
    \end{align*}
\end{definition}

Under standard multicalibration, we can bound the probability of $\Bad_{c,\beta}(\tf)$ under $\fD$ for each $c \in \mC$ by an application of Markov's inequality.
Note that the set $\Bad_{c,\beta}(\tf)$ depends on our choice of $c \in \mC$.
That is, different prediction intervals might be bad for different choices of $c$.
In principle, one might hope to take a union bound over all $c \in \mC$ to bound the probability of landing in a globally-bad set $\Bad_{\mC,\beta}(\tf)$.
But for large (or infinite) hypothesis classes $\mC$, the resulting bound becomes vacuous quickly.

In contrast, swap multicalibration allows us to define a global set $\Bad_{\mC,\beta}(\tf)$ of bad prediction intervals $v \subseteq \Im(\tf)$ with small probability under $\fD$.
That is, with good probability under $\fD$, the correlation between \emph{every} $c\in \mC$ and the residue $\y - v$ conditioned on $\tf(\x) =v$ will be small.
This difference is summarized in the following lemma.

\begin{lemma}
    If $\tf$ is $(\mC, \alpha)$-multicalibrated, then for every $c \in \mC$, 
    \[ \Pr_{\bv \sim \fD} [\bv \in \Bad_{c,\beta}(\tf)] \leq \frac{\alpha}{\beta}.\]
    If $\tf$ is $(\mC, \alpha)$-swap multicalibrated, then 
    \[ \Pr_{\bv \sim \fD} [\bv \in \Bad_{\mC,\beta}(\tf)] \leq \frac{\alpha}{\beta}.\]
\end{lemma}

\paragraph{Standard versus swap multicalibration.}
In definition, swap multicalibration is a stronger notion that than standard multicalibration.
That said, we show that, both analytically and algorithmically, the two notions are tightly connected.
In particular, any method (existing and hypothetical) for achieving (standard) $\mC$-multicalibration can be used to achieve $\mC$-swap multicalibration without significant deterioration in the accuracy parameter.

First, we show that
any predictor $\tf$ that guarantees standard multicalibration can be converted into a predictor $\bar{p}$ that guarantees swap multicalibration, using a simple bucketing procedure.
This bucketing causes some degradation in the multicalibration error parameter, but the class $\mC$ remains unaffected.
\begin{definition}
    Let $\delta \in [0,1]$ so that $m = 1/\delta \in \Z$. Define $B_j = [(j-1)\delta, j \delta)$ for $j \in [m-1]$ and $B_m = [1 - \delta, 1]$. Define the predictor $\bar{p}_\delta$ where for every $x$ such that $\tf(x) \in B_j$, $\bar{p}_\delta(x) = j\delta$.  
\end{definition}

\begin{claim}
\label{claim:to-uniform}
Let $\tf:\X \rgta [0,1]$ be a $(\mC, \alpha)$-multicalibrated.
    For any $\delta \in [0,1]$ where $1/\delta \in \Z$, the predictor $\bar{p}_\delta$ is $(\mC, 2\sqrt{\alpha/\delta} + \delta)$-swap multicalibrated, and $\max_{x \in \X}|\tp(x) - \bar{p}_\delta(x)| \leq \delta$.
\end{claim}
While the bound of the theorem is valid for all $\delta$, the swap multicalibration guarantee is only meaningful when $\delta \geq \Omega(\alpha)$.
Thus, by discretizing a $(\mC,\alpha)$-multicalibrated predictor, we obtain a $(\mC,O(\alpha^{1/3}))$-swap multicalibrated predictor.
Note that the resulting predictor $\bar{p}$ is actually close to the original $\tf$ in $\ell_\infty$; in other words, any (standard) multicalibrated predictor is actually close to a swap multicalibrated predictor.

\begin{proof}[Proof of Claim~\ref{claim:to-uniform}]
The $\ell_\infty$ bound is immediate from the definition of $\bar{p}$. We bound the swap multicalibration error of $tf$. 
We have $\bar{p}(\x) = j\delta$ iff $\tf(\x) \in B_j$, so that $|\tf(\x) - j\delta| \leq \delta$ holds conditioned on this event. So 
\begin{align}
    \umce_\mD(\bar{p}, \mC) & = \sum_{j \in [m]}\Pr[\bar{p}(\x) =j\delta]\max_{c \in \mC} \lt|\E_{\mD} [c(\x)(\y - j\delta)|\bar{p}(\x) =j \delta]\rt| \notag\\
    &  = \sum_{j \in [m]}\Pr[\tf(\x) \in B_j]\max_{c \in \mC} \lt|\E_{\mD} [c(\x)(\y - j\delta)|\tf(\x) \in B_j]\rt| \notag\\
    & \leq \sum_{j \in [m]}\Pr[\tf(\x) \in B_j]\lt(\delta + \max_{c \in \mC} \lt|\E_{\mD} [c(\x)(\y - \tf(\x))|\tf(\x) \in B_j]\rt| \rt)\notag\\
    &\leq \delta + \sum_{j \in [m]}\Pr[\tf(\x) \in B_j]\max_{c \in \mC} \lt|\E_{\mD} [c(\x)(\y - \tf(\x))|\tf(\x) \in B_j]\rt|    \label{eq:gap-delta}
\end{align}

Let us fix a bucket $B_j$  and a particular $c \in \mC$. For $\beta \geq \alpha$ to be specified later we have
\begin{align*}
\lt|\E[ c(\x)(\y - \tf(\x))|\tf(\x) \in B_j]\rt| & \leq  \Pr[c(\x)(\y - \tf(\x)) \geq \beta |\tf(\x) \in B_j] + \beta \Pr[c(\x)(\y - \tf(\x)) \leq \beta |\tf(\x) \in B_j] \\ 
&\leq \frac{\Pr[\tf(\x) \in \Bad_{c,\beta}(\tf) \cap B_j]}{\Pr[\tf(\x) \in B_j]} + \beta\\
&\leq \frac{\Pr[\tf(\x) \in \Bad_{c,\beta}(\tf)]}{\Pr[\tf(\x) \in B_j]} + \beta\\
&\leq \frac{\alpha/\beta}{\Pr[\tf(\x) \in B_j]} +\beta.
\end{align*}
Since this bound holds for every $c$, it holds for the max over $c \in  \mC$ conditioned on $\tf(\x) \in B_j$. Hence
\begin{align*}
 \sum_{j \in [m]}\Pr[\tf(\x) \in B_j]\max_{c \in \mC} \lt|\E_{\mD} [c(\x)(\y - \tf(\x))|\tf(\x) \in B_j]\rt|    & \leq \sum_{j \in [m]}\Pr[\tf(\x) \in B_j] \lt(\frac{\alpha/\beta}{\Pr[\tf(\x) \in B_j]} +\beta\rt)\notag\\
 & \leq \frac{\alpha}{\beta\delta} + \beta,
\end{align*}
where we use $m = 1/\delta$. Plugging this back into Equation \eqref{eq:gap-delta} gives
\begin{align*}
  \umce_\mD(\bar{p}, \mC)  = \E_{\bv \sim \bar{p}_\mD} \lt[ \max_{c \in \mC} \lt|\E_{\Dv} [c(\x)(\y - \bv)]\rt|\rt] \leq     \frac{\alpha}{\beta\delta} + \beta + \delta.
\end{align*}
Taking $\beta = \sqrt{\alpha/\delta}$ gives the desired claim.
\end{proof}

While this generic transformation suffers a polynomial loss in the accuracy parameter, such a loss seems not to be necessary algorithmically.
As we describe next, known algorithms for achieving multicalibration \cite{hkrr2018, omni} actually guarantee swap multicalibration without any modification.

\subsection{Algorithms for Swap Multicalibration}
In Lemma~\ref{lem:algo}, we show that swap multicalibration is feasible, assuming an oracle for the agnostic learning problem on $\mC$.
First, we give a brief overview of the prior work on algorithms for standard multicalibration.

The work introducing multicalibration \cite{hkrr2018} gives a boosting-style algorithm for learning multicalibrated predictors that has come to be known as $\mathrm{MCBoost}$.
The algorithm is an iterative procedure:  starting with a trivial predictor, the $\mathrm{MCBoost}$ searches for a supported value $v \in \Im(\tp)$ and ``subgroup'' $c_v \in \mC$ that violate the multicalibration condition.
\begin{algorithm}[t]
\caption{MCBoost \cite{hkrr2018}}\label{alg:mcboost}
\textbf{Parameters:} hypothesis class $\mC$ and $\alpha > 0$\\
\textbf{Given:}  Dataset $S$ sampled from $\mD$\\
\textbf{Initialize:}  $\tp(x) \gets 1/2$.\\
\textbf{Repeat:}\\
if $\exists v \in \Im(\tp)$ and $c_v \in \mC$ such that
\begin{gather}
\label{algo:searchstep}
\E[ c_v(\x) \cdot (\y - v) \ \vert \ \tp(\x) = v] > \poly(\alpha)
\end{gather}
update $\tp(x) \gets \tp(x) + \eta c_v(x) \cdot \mathbf{1}[\tp(x) = v]$\\
\textbf{Return:} $\tp$
\end{algorithm}
Note that some care has to be taken to ensure that the predictor $\tp$ stays supported on finitely many values, and that each of these values maintains significant measure in the data distribution $\mD_{\tp}$.
In this pseudocode, we ignore these issues; \cite{hkrr2018} handles them in full detail.

Importantly, the search over $\mC$ for condition (\ref{algo:searchstep}) can be reduced to weak agnostic learning, defined formally below.
In each iteration, we pass the learner samples drawn from the data distribution, but labeled according to $\z = \y - v$ when $\tp(\x) = v$.
\begin{definition}[Weak agnostic learning]
    Suppose $\mD$ is a data distribution supported on $\X \times [-1,1]$.
    For a hypothesis class $\mC$, a weak agnostic learner $\mathrm{WAL}$ solves the following promise problem:
    for some accuracy parameter $\alpha > 0$,
    if there exists some $c \in \mC$ such that
    \begin{gather*}
        \E_{(\x,\z) \sim \mD}[c(\x) \cdot \z] \ge \alpha
    \end{gather*}
    then $\mathrm{WAL}_\alpha$ returns some $h:\X \to \R$ such that\footnote{We are informal about the polynomial factor in the guarantee of the weak agnostic learner. The smaller the exponent, the stronger the learning guarantee; i.e., we want $\mathrm{WAL}_\alpha$ to return a hypothesis with correlation with $\z$ as close to $\Omega(\alpha)$ as possible.}
    \begin{gather*}
        \E_{(\x,\z) \sim \mD}[h(\x) \cdot \z] \ge \mathrm{poly}(\alpha).
    \end{gather*}
\end{definition}
Weak agnostic learning is a basic supervised learning primitive used in boosting algorithms.
Through the connection to boosting, weak agnostic learning is polynomial-time equivalent to agnostic learning, and inherits its statistical efficiency (scaling with the VC-dimension of $\mC$), and also its worst-case computational hardness \cite{pietrzak2012cryptography}.

$\mathrm{MCBoost}$ reduces the problem of swap multicalibration to the standard weak agnostic learning task.
It is not hard to see that while $\mathrm{MCBoost}$ was designed to guarantee multicalibration, it actually guarantees swap multicalibration, using an oracle for weak agnostic learning.
By this reduction, we establish the feasability of $(\mC,\alpha)$-swap multicalibration for any learnable $\mC$.
\begin{lemma}
\label{lem:algo}
    Suppose the collection $\mC$ has a weak agnostic learner, $\mathrm{WAL}$.
    For any $\alpha > 0$, $\mathrm{MCBoost}$ makes at most $\mathrm{poly}(1/\alpha)$ calls to $\mathrm{WAL}_{\alpha}$, and returns a $(\mC,\alpha)$-swap multicalibrated predictor.
\end{lemma}
\begin{proof}
The iteration complexity of $\mathrm{MCBoost}$ is inverse polynomially related to the size of the violations we discover in (\ref{algo:searchstep}).
A standard potential argument and the formal reduction of search to weak agnostic learning can be found in \cite{hkrr2018}.

By the termination condition, we can see that $\tp$ must actually be $(\mC,\alpha)$-swap multicalibrated.
In particular, when the algorithm terminates, then for all $v \in \Im(\tp)$, we have that
\begin{gather*}
    \max_{c_v \in \mC}\ \E[ c_v(\x) \cdot (\y - v) \ \vert \ \tp(\x) = v] \le \poly(\alpha) \le \alpha.
\end{gather*}
Therefore, averaging over $\bv \sim \mD_{\tp}$, we obtain the guarantee.
\end{proof}

\section{Swap Agnostic Learning, Swap Omniprediction, and Swap Multicalibration are Equivalent}
\label{sec:main}

Our main result is an equivalence between swap agnostic learning, swap omniprediction, and swap multicalibration.
Concretely, this equivalence shows that swap agnostic learning for the squared error is sufficient to guarantee swap omniprediction for all (nice) convex loss functions.
We begin with some preliminaries, then formally state and prove the equivalence.
We conclude the section by showing how to use the existing framework for learning multicalibrated predictors to achieve swap agnostic learning for any convex loss.

\paragraph{Nice loss functions.}
For a loss function $\ell : \zo \times \R \rgta \R$,
we extend  $\ell$ linearly to allow the first argument to take values in the range $p \in [0,1]$ as:
\[ \ell(p, t) = \E_{\y \sim \Ber(p)} [\ell(\y,t)] = p \cdot \ell(1, t) + (1- p) \cdot \ell(0,t). \]
We say the loss function is {\em convex} if for $y \in \zo$, $\ell(y, t)$ is a convex function of $t$.
By linearity, this convexity property holds for $\ell(p, t)$ for all $p$. 

As in \cite{lossoi}, for a loss $\ell$, we define the partial difference function $\partial \ell: \R \to \R$ as
    \[ \Pell (t) = \ell(1, t) - \ell(0, t).\]
We define a class of ``nice'' loss functions, which obey a minimal set of boundedness and Lipschitzness conditions.
\begin{definition}    
\label{def:nice-loss}
    For a constant $B > 0$, a loss function is $B$-nice if there exists an interval $I_\ell \subseteq \R$ such that the following conditions hold:
    \begin{enumerate}
        \item (Optimality) If $\Pi_\ell:\R \rgta I_\ell$ denotes projection onto the interval $I_\ell$, then $\ell(p, \Pi_\ell(t)) \leq \ell(p, t)$ for all $t \in \R$ and $p \in [0,1]$.
        \item (Lipschitzness) For $y \in \zo$ and $t \in I_\ell$, $\ell(y,t)$ is $1$-Lipschitz as a function of $t$. 
        \item (Bounded difference) For $t \in I_\ell$, $|\partial \ell(t)| \leq B$.
    \end{enumerate}
    The class of all $B$-nice loss functions by $\mL(B)$. The subset of $B$-nice convex loss functions is denoted by $\mLc(B)$.
\end{definition}
Bounded loss functions generalize the idea of loss functions defined over a fixed interval of $\R$ (by optimality) and of bounded output range (by bounded difference).
By optimality of nice loss functions, we may assume $k_\ell:\izo \to I_\ell$ since we do not increase the loss by projection onto $I_\ell$.
Indeed, for convex losses $\ell$, the natural choice for $I_\ell$ is to take the interval $[k_\ell(0), k_\ell(1)]$.  
The bounded difference condition implies that $\ell$ is Lipschitz in its first argument, a property that will be useful.
\begin{lemma}
\label{lem:y-lipschitz}
    For every $\ell \in \mL(B)$ and $t_0 \in I_\ell$, the function $\ell(p,t_0)$ is $B$-Lipschitz w.r.t.\ $p$.   
\end{lemma}

\paragraph{Concept classes.}
For a concept class of functions $\mC: \{c: \X \to \R\}$, we assume that $\mC$ is closed under negation, and it contains the constant functions $0$ and $1$.
Denoting
$\infnorm{\mC} = \max|c(x)|$ over $c \in \mC, x \in \X$, we say that $\mC$ is bounded if $\infnorm{c} \leq 1$.
For $W \in \R^+$, let $\Lin(\mC,W)$ be all functions that can be expressed as a ($W$-sparse) linear combination of base concepts from $\mC$,
\[ c_w(x) = \sum_{c \in \mC}w_c \cdot c(x), \ \ \sum_{c \in \mC} |w_c| \leq W.\]
Note that for bounded $\mC$, the norm of linear combinations scales gracefully with the sparsity, $\infnorm{\Lin(\mC, W)} \leq W$.
We define $\Lin(\mC)$ to be the set of all linear combinations with no restriction on the weights of the coefficients.

\subsection{Statement of Equivalence}

With these notions in place, we can formally state the main result.

\begin{theorem}
\label{thm:main}
Let $\tf$ be a predictor, $\mC$ be a bounded hypothesis class, and $\mLc(B)$ be the class of $B$-nice convex loss functions. The following properties are equivalent:\footnote{
When we say these conditions are equivalent, we mean that they imply each other with parameters $\alpha_i$ that are polynomially related. The relations we derive result in at most a quadratic loss in parameters.}
    \begin{enumerate}
        \item $\tf$ is $(\mC, \alpha_1)$-swap multicalibrated.
        \item $\tf$ is an $(\mLc(B), \Lin(\mC, W), O((W+B)\alpha_2))$-swap  omnipredictor, for all $W \geq 1, B \geq 0$.
        \item $\tf$ is an $(\ell_2, \Lin(\mC,2), \alpha_3)$-swap agnostic learner.
    \end{enumerate}
\end{theorem}

\subsubsection{Preliminary Results}
In preparation for proving the theorem, we establish some preliminary results.
We define a function $\alpha: \Im(\tf) \rgta [-1,1]$ which measures the maximum correlation between $c \in \mC$ and $\y - v$, conditioned on a prediction value $v \in \Im(\tf)$. Let
\[ \alpha(v) = \lt| \max_{c_v \in \mC} \E_{\mD|_v}[c_v(\x)( \y - v)] \rt|.\]
Using this notation, $(\mC, \alpha_0)$-swap multicalibration can be written as 
\[ \E_{\bv \sim \fD}[\alpha(\bv)] \leq \alpha_0.\]

We observe that  swap multicalibration is closed under bounded linear combinations of $\mC$, like with standard multicalibration. 
\begin{claim}
\label{lem:linearity}
For every $h \in \Lin(\mC, W)$ and $v \in \Im(\tf)$, we have
    \[ \max_{h \in \Lin(\mC, W)} \lt| \E_{\mD} [h(\x)(\y - v)|\tf(\x) = v]\rt| \leq W\alpha(v). \]
Let $p^*_v = \E[\y|\tf(\x) = v]$. Then  
    \begin{align}
    \label{eq:close}
        |p^*_v - v| \leq \alpha(v).
    \end{align}
\end{claim}
\begin{proof}
Suppose that $h \in \mathrm{Lin}(\mC,W)$ of the form $h(x) = \sum_{c \in \mC} w_c \cdot c(x)$.
From Claim~\ref{claim:swap-mc}, we know that the multicalibration violation for $c \in \mC$ is bounded by $\alpha(v)$ for every $v \in \Im(\tp)$.
\begin{align*}
    \left \vert \E[h(\x)(\y - v) \ \vert \ \tp(\x) = v] \right \vert
    &= \left \vert \E\left[\sum_{c \in \mC}w_c \cdot c(\x)(\y - v) \ \vert \ \tp(\x) = v \right ] \right \vert\\
    &\le \left(\sum_{c \in \mC} \vert w_c \vert \right) \cdot \max_{c \in \mC} \left \vert \E[c(\x)(\y - v) \ \vert \ \tp(\x) = v] \right \vert\\
    &\le W \cdot \alpha(v)
\end{align*}
The inequalities follow by Holder's inequality and the assumed bound on the weight of $W$ for $h \in \mathrm{Lin}(\mC,W)$.

Equation (\ref{eq:close}) follows since $1 \in \mC$.
\end{proof}

\begin{claim}
\label{claim:cov}
For $h \in \Lin(\mC, w)$, $v \in \Im(\tf)$ and $y \in \zo$, define the following conditional expectations:
\begin{align*}
    \mu(h:v) &= \E[h(\x)|\tf(\x) = v]\\
    \mu(h: v, y) &= \E[h(\x)|(\tf(\x), \y) = (v, y)].
\end{align*}
Then for each $y \in \zo$
\begin{align}
\label{eq:cond-ex}
 \Pr[\y =y|\tp(\x) =v]\lt| \mu(h:v,y)  - \mu(h:v) \rt|  \leq (W+1)\alpha(v).
\end{align}
\end{claim}
\begin{proof}
Recall that $\Cov[\y, \z] = \E[\y\z] - \E[\y]\E[\z]$. 
For any $h \in \Lin(\mC,W)$ we have
\begin{align*}
\lt|\Cov[\y, h(\x)|\tf(\x) = v] \rt| &= \lt| \E[h(\x)(\y - \E[\y])|\tf(\x) = v]\rt|\\
&= \lt|\E[h(\x)(\y - v)|\tf(x) =v]\rt| + \lt|\E[(v - \y)|\tf(\x) = v]\rt|\\
& \leq (W+1)\alpha(v)
\end{align*}
where we use the fact that $h \in \Lin(\mC, W)$ and $1 \in \mC$. 
Since $\y \in \zo$, this implies  the claimed bounds by standard properties of covariance (see \cite[Corollary 5.1]{omni}).
\end{proof}

Next we show the following lemma, which shows that one can replace $h(\x)$ by the constant $\Pi_\ell(\mu(h:v))$ without a large increase in the loss.

\begin{lemma}
\label{main:tech1}
 For all $h \in \Lin(\mC, W)$, $v \in \Im(\tf)$ and loss $\ell \in \mLc(B)$, we have
\begin{align}
    \label{eq:main_tech1}
    \E_{\mD|_v}[\ell(\y, \Pi_\ell(\mu(h:v))] \leq \E_{\mD|_v}[\ell(\y, h(\x))] + 2(W +1)\alpha(v).
\end{align}
\end{lemma}
\begin{proof}
    For any $y \in \zo$, 
    \begin{align}
    \E_{\mD|_v}[\ell(\y, h(\x))|(\tf(\x), \y) = (v, y)] &=  \E_{\mD|_v}[\ell(y, h(\x))| (\tf(\x), \y) = (v, y)] \notag\\
    & \geq \ell(y, \E[h(\x)|(\tf(\x), \y) = (v, y)])\label{eq:jensen}\\
    & = \ell(y, \mu(h: v, y))\notag\\
    & \geq \ell(y, \Pi_\ell(\mu(h: v, y)))\label{eq:proj}.
    \end{align} 
    where Equation \eqref{eq:jensen} uses Jensen's inequality, and Equation \eqref{eq:proj} uses the optimality of projection for nice loss functions. 
    Further, by the $1$-Lipschitzness of $\ell$ on $I_\ell$, and of $\Pi_\ell$ on $\R$ 
    \begin{align}
        \ell(y, \Pi_\ell(\mu(h: v, y))) - \ell(y, \Pi_\ell(\mu(h:v))) & \leq \lt|\Pi_\ell(\mu(h: v, y)) - \Pi_\ell(\mu(h:v))\rt| \notag\\
        & \leq \lt|\mu(h: v, y) - \mu(h:v)\rt| \label{eq:lip}
    \end{align}
    Hence we have
    \begin{align*}
    & \E_{\mD|_v}[\ell(\y, \Pi_\ell(\mu(h:v))] - \E_{\mD|_v}[\ell(\y, h(\x))]\\
    &= \sum_{y \in \zo}\Pr[\y =y|\tp(\x) = v]\lt(\ell(y, \Pi_\ell(\mu(h:v))) - \E[\ell(y, h(\x))|(\tf(\x), \y) = (v, y)]\rt)\notag\\
    &\leq \sum_{y \in \zo}\Pr[\y =y|\tp(\x) = v]\lt(\ell(y, \Pi_\ell(\mu(h:v))) -  \ell(y, \Pi_\ell(\mu(h:v, y)))\rt)\tag{By Equation \eqref{eq:proj}}\\
    & \leq \sum_{y \in \zo}\Pr[\y =y|\tp(\x) = v]\lt|\mu(h:v, y) - \mu(h:v)\rt|\tag{by Equation \eqref{eq:lip}}\\
    & \leq 2(W +1)\alpha(v). \tag{By Equation \eqref{eq:cond-ex}}
    \end{align*}
\end{proof}

Next we compare $\Pi_\ell(\mu(h:v))$ with $k_\ell(v)$. It is clear that the latter is better for minimizing loss when $\y \sim \Ber(v)$, by definition. We need to compare the losses when $\y \sim \Ber(p^*_v)$.  But $p^*_v$ and $v$ are at most $\alpha(v)$ apart by Equation \eqref{eq:close}. Hence, by using Lipschitzness, one can infer that $k_\ell(v)$ is better than $\Pi_\ell(\mu(h:v))$ and hence $h(\x)$. This is formalized in the following lemma and its proof.

\begin{lemma}
\label{lem:main_tech2}
For all $v \in \Im(\tf)$,  $\ell \in \mLc(B)$ and $h \in \Lin(\mC, W)$, we have
    \begin{align} 
    \label{eq:main_tech2}
    \E_{\mD|_v}[\ell(\y, k_\ell(\tf(\x)))] \leq \E_{\mD|_v}[\ell(\y, h(\x))] + 2(W + B + 1)\alpha(v).
    \end{align}
\end{lemma}
\begin{proof}
    By the definition of $k_\ell$, $k_\ell(v)$ minimizes expected loss when $\y \sim \Ber(v)$, so
    \begin{align} 
        \label{eq:min_kl}
        \ell(v, k_\ell(v)) \leq \ell(v, \Pi_\ell(\mu(h:v))
    \end{align}
    On the other hand, 
    \[ \E_{\mD|_v}[\ell(\y, t)] = \ell(p^*_v, t), \ \text{where} \ p^*_v = \E_{\mD|_v}[\y].\] 
    Thus our goal is compare the losses $\ell(p^*_v, t)$ for $t = k_\ell(v)$ and $t = \Pi_\ell(\mu(h:v))$.
    Hence, applying Lemma \ref{lem:y-lipschitz} gives
    \begin{align*}
    \ell(p^*_v, k_\ell(v)) & \leq  \ell(v, k_\ell(v)) + \alpha(v)B\\
    - \ell(p^*_v, \Pi_\ell(\mu(h:v)) & \leq  -\ell(v, \Pi_\ell(\mu(h:v)) + \alpha(v)B
    \end{align*}
    Subtracting these inequalities and then using Equation \eqref{eq:min_kl} gives
    \begin{align} 
    \ell(p^*_v, k_\ell(v)) - \ell(p^*_v, \Pi_\ell(\mu(h:v)) &\leq  \ell(v, k_\ell(v)) - \ell(v, \Pi_\ell(\mu(h:v)) + 2B\alpha(v) \\
    & \leq 2\alpha(v)B. \label{eq:missing}
    \end{align}
    We can now write
    \begin{align*} 
    \E_{\mD|_v}[\ell(\y, k_\ell(v))] & = \ell(p^*_v, k_\ell(v))\\
    & \leq\ell(p^*_v, \Pi_\ell(\mu(h:v)) + 2\alpha(v) B\tag{by Equation \eqref{eq:missing}}\\
    & = \E_{\mD|_v}[\ell(\y, \Pi_\ell(\mu(h:v))] + 2\alpha(v)B \\
    & \leq \E_{\mD|_v}[\ell(\y, h(\x))] + 2(W +1)\alpha(v) +2B\alpha(v). \tag{by Equation \eqref{eq:main_tech1}}
    \end{align*}
\end{proof}

\subsubsection{Proof of Equivalence}
We now complete the proof of Theorem \ref{thm:main}.

\begin{proof}[Proof of Theorem \ref{thm:main}]
(1) $\implies$ (2) Fix a $(\mC, \alpha_1)$-swap multicalibrated predictor $\tf$. Fix a choice of loss functions $\{\ell_v \in \mLc\}_{v \in \Im(f)}$ and hypotheses $\{h_v \in \mH\}_{v \in \Im(f)}$. For each $v$, we apply Equation \eqref{eq:main_tech2} with the loss $\ell =\ell_v$, hypothesis $h = h_v$ to get
 \begin{align*} 
    \E_{\mD|_v}[\ell_v(\y, k_{\ell_v}(v))] \leq \E_{\mD|_v}[\ell_v(\y, h_v(\x))] + 2(W + B + 1)\alpha(v).
\end{align*}
We now take expectations over $\bv \sim \fD$, and use the $\E[\alpha(\bv)] \leq \alpha_1$ to derive the desired implication.

(2) $\implies$ (3) with $\alpha_3 = 7\alpha_2$ because $\ell_2$ is a $1/2$-nice loss function, so we plug in $B = 1/2$ and $W =2$ into claim (2). 

It remains to prove that $(3) \implies (1)$. We show the contrapositive, that if $\tf$ is not $\alpha_1$ multicalibrated, then $f = \tf$ is not a $(\ell_2, \Lin(\mC, 2), \alpha_3)$-swap agnostic learner.  By the definition of multicalibration, for every $v \in \Im(\tf)$, there exist $c_v$ such that
\begin{align*}
    \E_{\mD|_v}[c_v(\x)(\y - \tf(\x))] &= \alpha(v)\\
    \E_{\fD}[\alpha(\bv)] &\geq \alpha_1. 
\end{align*}
By negating $c_v$ if needed, we may assume $\alpha(v) \geq 0$ for all $v$. We now define the updated hypothesis $h'$ where
\[ h'(x) = v + \alpha(v) c_v(x) \ \text{for} \ x \in \tf^{-1}(v) \]

A standard calculation (included in the Appendix) shows that
\begin{align*} 
\E_{\mD|_v}[(\y - v)^2] - \E_{\mD|_v}[(\y - h'(\x))^2]
\geq  \alpha(v)^2.
\end{align*}
Taking expectation over $\bv \sim \fD$, we have
\begin{align*} 
\E_{\fD}\lt(\E_{\Dv}[(\y - v)^2] - \E_{\Dv}[(\y - h'(\x))^2] \rt) & \geq \E_{\fD}[\alpha(\bv)^2] \geq \E_{\fD}[\alpha(\bv)]^2 \geq \alpha_1^2.
\end{align*}

It remains to show that $v + \alpha_vc_v(x) \in \Lin(\mC, 2)$. Note that 
\[ \alpha_v = \E_{\mD|_v}[c(\x)(\y- v)v] \leq \max|c(\x)| \max(|y - \tf(x)| \leq 1 \] 
since $c(\x), y - v \in [-1, 1]$. 
Hence $h'(x) =  w_1\cdot 1 + w_2 c(v)$ where $|w_1| + |w_2| \leq 2$. This contradicts the definition of an $(\ell_2, \Lin(\mC,2), \alpha_3)$-swap agnostic learner if $\alpha_3 < \alpha_1^2$. 
\end{proof}

\subsection{An algorithm for Swap Agnostic Learning}

The equivalence from Theorem~\ref{thm:main} suggests an immediate strategy for obtaining a swap agnostic learner.
First, learn a swap multicalibrated predictor; then, return the predictor, post-processed to an optimal hypothesis according to $\ell$.
Importantly, $\mathrm{MCBoost}$ reduces the problem of swap multicalibration (and thus, swap agnostic learning) to a standard agnostic learning task.
In all, we can combine the $\mathrm{MCBoost}$ algorithm for a class $\mC$ with a specific loss $\ell$ to obtain a $(\ell,\mC)$-swap agnostic learner.
\begin{corollary}[Informal]
\label{cor:sal}
For any (nice) convex loss $\ell$, hypothesis class $\mC$, and $\eps > 0$,
Algorithm~\ref{alg:sal} returns a $(\ell,\mC,\eps)$-swap agnostic learner from a sample of $m \le \mathrm{VC}(\mC) \cdot \poly(1/\eps)$ data points drawn from $\mD$, after making $\le \poly(1/\eps)$ calls to weak agnostic learner for $\mC$.
\end{corollary}

\begin{algorithm}[t!]
\caption{Swap Agnostic Learning via MCBoost}\label{alg:sal}
\textbf{Parameters:} loss $\ell$, hypothesis class $\mC$, and $\eps > 0$, let $\alpha = \mathrm{poly}(\eps)$\\
\textbf{Given:}  Dataset $S$ sampled from $\mD$\\
\textbf{Run:}\\
$\tp \gets \mathrm{MCBoost}_{\mC,\alpha}(S)$\\
$h_\ell \gets k_\ell \circ \tp$\\
\textbf{Return:} $h_\ell$
\end{algorithm}

\section{Characterizing Swap Loss OI via Multicalibration}
\label{sec:lossOI}

We show that $(\mL, \mC)$-swap loss OI and $(\partial \mL \circ \mC)$-swap multicalibration are equivalent for nice loss functions. 

\begin{theorem}
\label{thm:swap-eq}
Let $\mL \subseteq \mL(B)$ be a family of $B$-nice loss functions such that $\ell_2 \in \mL$. Then $(\partial \mL \circ \mC, \alpha_1)$-swap multicalibration and $(\mL, \mC, \alpha_2)$-swap loss OI are equivalent.\footnote{Here equivalence means that there are reductions in either direction that lose a multiplicative factor of $(B + 1)$ in the error.}
\end{theorem}

In preparation for this, we first give the proof of Lemma~\ref{lem:swap-strong2}.
\begin{proof}[Proof of Lemma \ref{lem:swap-strong2}]
The proof of Part (1) follows the proof of \cite[Proposition 4.5]{lossoi}, showing that loss OI implies omniprediction. By the definition of $k_{\ell_v}$, for every $x \in \X$ such that $\tp(x) = v$
\begin{align*} 
    \E_{\ty \sim \Ber(v)} u_v(x, v, \ty) &= \E_{\ty \sim \Ber(v)}[\ell_v(\ty, k_{\ell_v}(v)) -  \ell_v(\ty, c_v(x))]\\
    &= \ell_v(v, k_{\ell_v}(v)) -  \ell_v(v, c_v(x))\\
    & \leq 0
\end{align*}
Hence this also holds in expectation under $\mD|_v$, which only considers points where $\tp(\x) =v$:
\[ \E_{\mD|_v} [u_v(\x, v, \ty)] \leq 0. \]

Since $\tf$ satisfies swap loss OI, we deduce that
\begin{align*} 
\E_{\mD|_v}[u_\bv(\x, v, \y^*)]  \leq \alpha(v)
\end{align*}
Taking expectations over $\bv \sim \fD$ and using the definition of $u_v$, we get 
\begin{align*}
 \E_{\bv \sim  \fD} \E_\Dv[\ell_\bv(\y^*, k_{\ell_\bv}(\bv)) - \ell_\bv(\y^*, c_\bv(\x))] &= \E_{\bv \sim \fD}  \E_{\mD}[u_\bv(\x, \bv, \y^*)] \\
& \leq \E_{\bv \sim \fD}[\alpha(\bv)] \leq \alpha
\end{align*}
Rearranging the outer inequality gives
\[ \E_{\bv \sim \fD} \E_{\Dv} [\ell_\bv(\ty, k_{\ell_\bv}(\bv))] \leq  \E_{\bv \sim \fD}\E_{\Dv}[\ell_\bv(\y^*, c_\bv(\x))] +  \alpha.\]

Part (2) is implied by taking $\ell_v = \ell$ for every $v$. 
\end{proof}

\begin{proof}[Proof of Lemma \ref{lem:swap-strong2}]
The proof of Part (1) follows the proof of \cite[Proposition 4.5]{lossoi}, showing that loss OI implies omniprediction. By the definition of $k_{\ell_v}$, for every $x \in \X$ such that $\tp(x) = v$
\begin{align*} 
    \E_{\ty \sim \Ber(v)} u_v(x, v, \ty) &= \E_{\ty \sim \Ber(v)}[\ell_v(\ty, k_{\ell_v}(v)) -  \ell_v(\ty, c_v(x))]\\
    &= \ell_v(v, k_{\ell_v}(v)) -  \ell_v(v, c_v(x))\\
    & \leq 0
\end{align*}
Hence this also holds in expectation under $\mD|_v$, which only considers points where $\tp(\x) =v$:
\[ \E_{\mD|_v} [u_v(\x, v, \ty)] \leq 0. \]

Since $\tf$ satisfies swap loss OI, we deduce that
\begin{align*} 
\E_{\mD|_v}[u_\bv(\x, v, \y^*)]  \leq \alpha(v)
\end{align*}
Taking expectations over $\bv \sim \fD$ and using the definition of $u_v$, we get 
\begin{align*}
 \E_{\bv \sim  \fD} \E_\Dv[\ell_\bv(\y^*, k_{\ell_\bv}(\bv)) - \ell_\bv(\y^*, c_\bv(\x))] &= \E_{\bv \sim \fD}  \E_{\mD}[u_\bv(\x, \bv, \y^*)] \\
& \leq \E_{\bv \sim \fD}[\alpha(\bv)] \leq \alpha
\end{align*}
Rearranging the outer inequality gives
\[ \E_{\bv \sim \fD} \E_{\Dv} [\ell_\bv(\ty, k_{\ell_\bv}(\bv))] \leq  \E_{\bv \sim \fD}\E_{\Dv}[\ell_\bv(\y^*, c_\bv(\x))] +  \alpha.\]

Part (2) is implied by taking $\ell_v = \ell$ for every $v$. 
\end{proof}

We use the following simple claim from \cite{lossoi}.

\begin{claim}[Lemma 4.8, \cite{lossoi}]
    For random variables $\y_1, \y_2 \in \zo$ and $t \in \R$, 
    \begin{align}
    \label{eq:partial}
        \E[\ell(\y_1, t) - \ell(\y_2, t)] = \E[(\y_1 - \y_2)\Pell(t)].
    \end{align} 
\end{claim}

We record two corollaries of this claim. These can respectively be seen as strengthenings of the two parts of Theorem \cite[Theorem 4.9]{lossoi}, which respectively characterized hypothesis OI in terms of multiaccuracy and decision OI in terms of calibration. We generalize these to the swap setting.

\begin{corollary}
\label{cor:hyp-oi}
For every choice of $\{\ell_v, c_v\}_{v \in \Im(\tp)}$, we have
\begin{align}
\label{eq:umc}
    \E_{\bv \sim \fD} \lt[ \lt| \E_{\Dv} [\ell_\bv(\y^*, c_\bv(\x))- \ell_\bv(\ty, c_\bv(\x))] \rt| \rt] &= \E_{\bv \sim \fD} \lt[ \lt| \E_{\Dv}[(\y^* - \ty) \partial \ell_\bv \circ c_\bv(\x)] \rt|\rt].
\end{align}
Hence if $\tf$ is $(\partial \mL \circ \mC, \alpha)$-swap multicalibrated, then 
    \begin{align*}
        \E_{\bv \sim \fD} \lt[  \lt|\E_{\Dv} [\ell_\bv(\y^*, c_\bv(\x))- \ell_\bv(\ty, c_\bv(\x))] \rt| \rt]  \leq \alpha. 
    \end{align*}
\end{corollary}
\begin{proof}
Equation \eqref{eq:umc} is derived by applying Equation \eqref{eq:partial} to the LHS. Assuming that $\tf$ is $(\partial \mL \circ \mC, \alpha)$-swap multicalibrated, we have
\begin{align*}
 \E_{\bv \sim \fD} \lt[ \lt| \E_{\Dv}[(\y^* - \ty) \partial \ell_\bv \circ c_\bv(\x)] \rt|\rt] & \leq  \E_{\bv \sim \fD} \lt[ \lt| \max_{c' \in \partial \mL \circ \mC} \E_{\Dv}[(\y^* - \ty) c'(\x)] \rt|\rt] \leq \alpha.
\end{align*}
\end{proof}

\begin{corollary}
\label{cor:dec-oi}
    Let $\{\ell_v\}_{v \in \Im(f)}$ be a collection of loss $B$-nice loss functions. Let $k(v) = k_{\ell_v}(v)$. 
    If $\tp$ is $\alpha$-calibrated then 
    \begin{align}
        \E_{\bv \sim \fD}\lt[ \lt| \E_\Dv[\ell_\bv(\y^*, k(\bv)) - \ell_\bv(\ty, k(\bv))] \rt|\rt] & \leq  B\alpha. 
    \end{align}
\end{corollary}
\begin{proof}
    We have
    \begin{align*}
        \E_{\bv \sim \fD}\lt[ \lt| \E_\Dv[\ell_\bv(\y^*, k(\bv)) - \ell_\bv(\ty, k(\bv))] \rt|\rt] &= \E_{\bv \sim \fD}\lt[ \lt| \E_\Dv[(\y^* - \bv) \partial \ell_\bv(k(\bv))]\rt|\rt]\\
        &= \E_{\bv \sim \fD}\lt[  |\partial \ell_\bv(k(\bv))| \lt|\E_\Dv[\y^* - \bv]\rt|\rt]\\
        &\leq B\E_{\bv \sim \fD} \lt[ \lt| \E_\Dv[\y^* - \bv] \rt|\rt] \\
        & \leq B\alpha. 
        \end{align*}
        where we use the fact that $k(v) \in I_\ell$, and so $|\partial \ell_v(k(v))| \leq B$. 
\end{proof}

Finally, we show the following key technical lemma which explains why the $\ell_2$ loss has a special role.

\begin{lemma}
\label{lem:imp-cal}
    If $\tf$ is $(\{\ell_2\}, \mC, \alpha)$-swap OI, then it is $\alpha$-calibrated.
\end{lemma}
\begin{proof}
    Observe that $\ell_2(y, v) = (y -v)^2/2$ so $k_{\ell_2}(v) = v$. Hence, 
    \begin{align} 
    u_{\ell_2, 0}(x, v, y) &= \ell_2(y, k_\ell(v)) - \ell_2(y, 0)\notag\\
    & = ((y -v)^2 - y^2)/2 \notag\\
    &= -vy + v^2/2. \label{eq:v=0}
    \end{align}

    Recall that $\zo \subset \mC$. The implication of swap loss OI when we take $c_v = 0$ for all $v$ is that
    \begin{align*}
    \E_{\bv \sim \fD} \lt[\lt| \E_{\Dv}[u_{\ell_2, 0}(\x, \bv, \y^*)- u_{\ell_2, 0}(\x, \bv, \ty) ] \rt|\rt]  \leq \alpha.
    \end{align*}
    We can simplify the LHS using Equation \eqref{eq:v=0} to derive 
    \begin{align}
    \E_{\bv \sim \fD} \lt[\lt| \E_{\Dv}[(-\bv \y^* + \bv^2/2) - (-\bv \ty  + \bv^2/2)]\rt| \rt]
    &= \E_{\bv \sim \fD} \lt[\lt| \E_{\Dv} [\bv(\y^* - \ty)] \rt|\rt]\notag\\
    &= \E_{\bv \sim \fD} \lt[ \bv \lt|\E_{\Dv}[\ty - \y^*] \rt| \rt]  \leq \alpha. \label{eq:cons-0}
    \end{align}
    Considering the case where $c_v = 1$ for all $v$ gives
    \begin{align*}
        u_{\ell_2, 1}(x, v, y) &= \ell_2(y, k_\ell(v)) - \ell_2(y, 1)\\
        &= ((y -v)^2 - (1- y)^2)/2 \\
        &= (1-v)y +  (v^2-1)/2.    
    \end{align*}
    We derive the following implication of swap loss OI by taking $c_v = 0$ for all $v$:
    \begin{align}
    \E_{\bv \sim \fD} \lt[\lt| \E_{\Dv}[u_{\ell_2, 1}(\x, \bv, \y^*)- u_{\ell_2, 1}(\x, \bv, \ty) ] \rt|\rt]
    &= \E_{\bv \sim \fD} \lt[ (1 - \bv) \lt|\E_\Dv[\ty - \y^*] \rt| \rt] \leq \alpha \label{eq:cons-1}
    \end{align}
    Adding the  bounds from Equations \eqref{eq:cons-0} and \eqref{eq:cons-1} we get
     \begin{align*}
     \E_{\bv \sim \fD} \lt[ \lt| \E_\Dv[\bv - \y^*] \rt| \rt]  = \E_{\bv \sim \fD} \lt[ \lt| \E_\Dv[\ty - \y^*] \rt| \rt] \leq \alpha
     \end{align*}
\end{proof}

We can now complete the proof of Theorem \ref{thm:swap-eq}.

\begin{proof}[Proof of Theorem \ref{thm:swap-eq}]
We first show the forward implication, that swap multicalibration implies swap loss OI.

    Since $\ell_2 \in \mL$ and $1 \in  \mC$, we have $\partial \ell_2 \circ 1 = 1 \in \partial \mL \circ \mC$. This implies that $\tp$ is $\alpha$-mulitcalibrated, since 
    \[ \E_{\bv \sim \fD} \lt[ \lt|\E_{\Dv} [1(\y - \bv)]  \rt| \rt] \leq \E_{\bv \sim \fD} \lt[ \max_{c \in \mC} \lt|\E_{\Dv} [c(\x)(\y - \bv)]  \rt| \rt] \leq \alpha. \] 
    
    Consider any collection of losses $\{\ell_v\}_{v \in \Im(\tp)}$. Applying \Cref{cor:dec-oi}, we have
    \begin{align*}
    \E_{\bv \sim \fD}\lt[ \lt| \E_\Dv[\ell_\bv(\y^*, k(\bv)) - \ell_\bv(\ty, k(\bv))] \rt|\rt] & \leq B\alpha. 
    \end{align*}
    
    On the other hand, by Corollary \ref{cor:hyp-oi}, we have for every choice of $\{\ell_v, c_v\}_{v \in \Im(\tp)}$, 
    \begin{align*}
        \E_{\bv \sim \fD} \lt[  \lt|\E_{\Dv} [\ell_\bv(\y^*, c_\bv(\x))- \ell_\bv(\ty, c_\bv(\x))] \rt| \rt] \leq \alpha. 
    \end{align*}
    Hence for any choice of $\{u_v\}_{v \in \Im(\tf)}$ we can bound
    \begin{align*}
    & \E_{\bv \sim \fD} \lt| \E_{\Dv}[u_\bv(\x, \bv, \y^*)- u_\bv(\x, \bv, \ty)] \rt| \\
    & \leq \E_{\bv \sim \fD} \lt[ \lt| \E_{\Dv}[\ell_\bv(\y^*, k(\bv))- \ell_\bv(\ty, k(\bv)] \rt| + \lt|\E_{\Dv} [\ell_\bv(\y^*, c_\bv(\x))- \ell_\bv(\ty, c_\bv(\x))] \rt| \rt] \\
    & \leq (B+1)\alpha
\end{align*}  
which shows that $\tp$ satisfies swap loss OI with $\alpha_2 = (D +1)\alpha_1$. 

Next we show the reverse implication: if $\tp$ satisfies $(\mL, \mC, \alpha_2)$-swap loss OI, then it satisfies $(\partial \mL \circ \mC, \alpha_1)$-swap multicalibration. The first step is to observe that by \cref{lem:imp-cal}, since $\ell_2 \in \mL$, the predictor $\tf$ is $\alpha_2$ calibrated. Since any $\ell \in \mL$ is $B$-nice, we have 
\begin{align*} 
\E_{\bv \sim \fD}\lt[ \lt| \E_{\Dv}[ \ell_\bv(\y^*, k(\bv)) - \ell_\bv(\ty, k(\bv))] \rt|\rt] = \E_{\bv \sim \fD} \lt[ \lt| \E_{\Dv}[(\y^* - \ty) k(\bv)]\rt|\rt] \leq  B \alpha_2.
\end{align*}
For any $\{\ell_v, c_v\}_{v \in \Im(f)}$, since
\[ u_v(x, v, y) = \ell_v(y, k_\ell(v)) + \ell_v(y, c_v(x))\]
we can write
\begin{align*} 
& \E_{\bv \sim \fD} \lt[ \lt| \E_{\Dv}[ \ell_\bv(\y^*, c(\x)) - \ell_\bv(\ty, c(\x))] \rt| \rt] \\
& \leq \E_{\bv \sim \fD} \lt[ \lt| \E_{\Dv}[ u_{\bv}(\x, \bv, \y^*)  - u_{\bv}(\x, \bv, \ty)]\rt|  + \lt| \E_{\Dv}[ \ell_\bv(\y^*, k(\bv)) - \ell_\bv(\ty, k(\bv))] \rt| \rt]\\
& \leq (B + 1)\alpha_2.
\end{align*}
But by Equation \eqref{eq:umc}, the LHS can be written as
\[ \E_{\bv \sim \fD} \lt[ \lt| \E_{\Dv}[ \ell_\bv(\y^*, c(\x)) - \ell_\bv(\ty, c(\x))] \rt| \rt]  = \E_{\bv \sim \fD} \lt[ \lt| \E_{\Dv}[\partial \ell_\bv \circ c_\bv(\x) (\y^* -\bv)]  \rt| \rt]\]
This shows that $\tp$ is $(\partial \mL \circ \mC, (B+1)\alpha_2)$-swap multicalibrated.
\end{proof}

 \section{Relating Notions of Omniprediction}
\label{sec:example}

In this work, we have discussed the four different notions of omniprediction defined to date.
\begin{itemize}
\item[00)] Omniprediction, as originally defined by \cite{omni}.
\item[01)] Loss OI, from \cite{lossoi}.
\item[10)] Swap omniprediction.
\item [11)] Swap Loss OI.
\end{itemize}
In order to compare them, we can ask which of these notions implies the other for any fixed choice of loss class $\mL$ and hypothesis class $\mC$.
\begin{itemize}
\item Loss OI implies omniprediction by \cite[Proposition 4.5]{lossoi}. 
\item Swap omniprediction implies omniprediction by Claim~\ref{claim:swap-strong}.
\item Swap loss OI implies both loss OI and swap multicalibration by Lemma \ref{lem:swap-strong2}.
\end{itemize}
These relationships were summarized in Figure~\ref{fig:rel}.
Further, this picture captures all the implications that hold for all $(\mL, \mC)$.
Next, we show that for any implication not drawn in the diagram, there exists some (natural) choice of $(\mL, \mC)$, where the implication does not hold.
In particular, we prove that neither loss OI nor swap omniprediction implies the other for all $(\mL, \mC)$.
This separates these notions from swap loss OI, since swap loss OI implies both these notions.\footnote{For instance if loss OI implied swap loss OI, it would also imply swap omniprediction, which our claim shows it does not.}
By similar reasoning, it separates omniprediction from both these loss OI and swap omniprediciton, since omniprediction is implied by either of them.

\paragraph{Swap omniprediction does not imply loss OI.}
We prove this non-implication using a counterexample used in \cite{lossoi}.
In particular, they show that omniprediction does not imply loss OI \cite[Theorem 4.6]{lossoi}, and the same example in fact shows that swap omniprediction does not imply loss OI.
In their example, we have $\mD$ on $\pmo^3 \times [0,1]$ where the marginal on $\pmo^3$ is uniform, and $p^*(x) = (1 + x_1x_2x_3)/2$, whereas  $\tp(x) =1/2$ for all $x$. We take $\mC = \{1, x_1, x_2, x_3\}$.
Since $\tp = 1/2$ is constant, it is easy to check that $\tp - p^* = -x_1x_2x_3/2$ is uncorrelated with $\mC$.
Hence $\tp$ satisfies swap multicalibration (which is the same as multicalibration or even multiaccuracy in this setting where $\tp$ is constant).
Hence by Theorem \ref{thm:main}, $\tp$ is an $(\mLc(1), \mLC, 0)$-swap omnipredictor. 
\cite[Theorem 4.6]{lossoi} prove that $\tp$ is not loss OI for the $\ell_4$ loss.
Hence we have the following result.
\begin{lemma}
\label{lem:sep1}
    The predictor $\tp$ is $(\mC, 0)$-swap multicalibrated and hence it is a $(\{\ell_4\}, \Lin(\mC), 0)$-swap omnipredictor. But it is not $(\{\ell_4\}, \Lin(\mC, 1), \eps)$-loss OI for $\eps < 4/9$. 
\end{lemma}

We remark that the construction extends to all $\ell_p$ losses for even $p > 2$. Hence  even for convex losses, the notions of swap omniprediction are loss-OI seem incomparable. 

\paragraph{Loss OI does not imply swap omniprediction.}
Next we construct an example showing that loss OI need not imply swap omniprediction.
We consider the set of all GLM losses defined below, which contain common losses including the squared loss and the logistic loss.

\begin{definition}
\label{def:lglm}
    Let $g:\R \to \R$ be a convex, differentiable function such that $[0,1] \subseteq \Im(g')$. Define its matching loss to be $\ell_g = g(t) - yt$. Define $\lglm =\{\ell_g\}$ be the set of all such loss functions.
\end{definition}

\cite{lossoi} shows a general decomposition result that reduces achieving loss OI to a calibration condition and a multiaccuracy condition.
Whereas arbitrary losses might require multiaccuracy for the more powerful class $\partial \mL \circ \mC$, for $\lglm$ we have  $\partial \lglm \circ \mC = \mC$.
This is formalized in the following result. 

\begin{lemma}[Theorem 5.3, \cite{lossoi}]\label{lem:glm}
    If $\tf$ is $\eps_1$-calibrated and $(\mC, \eps_2)$-multiaccurate, then it is $(\lglm, \Lin(\mC,W), \delta)$-loss OI for $\delta  = \eps_1 + W\eps_2$.
\end{lemma}

In light of the above result, it suffices to find a predictor that is calibrated and multiaccurate (and hence satisfies loss OI), but not multicalibrated, hence not swap multicalibrated.
By Theorem \ref{thm:main} it is not an $(\{\ell_2\}, \mLC, \delta)$-swap omnipredictor for $\delta$ less than some constant.

We show a separation between loss OI and swap omniprediction using  the predictors $p^*, \tf: \pmo^2 \to \izo$ defined in Table \ref{tab:pred}.

\begin{table}
\centering
\begin{tabular}{||c | c| c ||} 
 \hline
 $x =(x_1, x_2)$ & $p^*(x)$ & $\tp(x)$  \\ [1ex] 
 \hline \hline
 $(-1, -1)$ & $0$ & $\fr{8}$ \\ [0.5ex]
 \hline
 $(+1, -1)$ & $\fr{4}$ & $\fr{8}$ \\[0.5ex]
 \hline
 $(-1, +1)$ & $1$ & $\frac{7}{8}$ \\[0.5ex]
 \hline
 $(+1,+1)$ & $\frac{3}{4}$ & $\frac{7}{8}$\\ [1ex]
 \hline
\end{tabular}
\label{tab:pred}
\caption{Separating loss-OI and swap-resilient omniprediction}
\end{table}

\begin{lemma}
Consider the distribution $\mD$ on $\pmo^2 \times \zo$ where the marginal on $\pmo^2$ is uniform and $\E[\y|x] = p^*(x)$. Let $\mC = \{1, x_1, x_2\}$. 
\begin{enumerate}
\item $\tf \in \Lin(\mC, 1)$. Moreover, it minimizes the squared error over all hypotheses from $\Lin(\mC)$.
\item $\tf$ is perfectly calibrated and $(\mC,0)$-multiaccurate. So it is $(\lglm, \Lin(\mC), 0)$-loss OI.
\item $\tf$ is not $(\mC, \alpha)$-multicalibrated for $\alpha < 1/8$. It is not $(\ell_2, \Lin(\mC), \delta)$-swap agnostic learner for $\delta < 1/64$.
\end{enumerate}
\end{lemma}
\begin{proof}
We compute Fourier expansions for the two predictors:
\begin{align}
    p^*(x) &= \fr{8}(4 + 3x_2 - x_1x_2)\label{eq:f*}\\
    \tf(x) &= \fr{8}(4 + 3x_2)\label{eq:f}
\end{align}
This shows that $\tf \in \Lin(\mC)$, and moreover that it is the optimal approximation to $p^*$ in $\Lin(\mC)$, as it is the projection of $p^*$ onto $\Lin(\mC)$. This shows that $\tf$ is an $(\ell_2, \Lin(\mC), 0)$-agnostic learner. 

It is easy to check that $\tf$ is perfectly calibrated.   It is $(\mC, 0)$-multiaccurate, since it is the projection of $p^*$ onto $\Lin(\mC)$, so $\tf - p^*$ is orthogonal to $\Lin(\mC)$. Hence we can apply Lemma \ref{lem:glm} to conclude that it is $(\lglm, \Lin(C), 0)$-loss OI, where $\lglm$ which contains the squared loss.

To show  that $\tf$ is not swap-agnostic, we observe that
conditioning on the value of $\tf(\x) = (4 + 3x_2)/8$ is equivalent to conditioning on $x_2 \in \pmo$. For each value of $x_2$, the restriction of $p^*$ which is now linear in $x_1$ belongs to $\Lin(\mC)$. Indeed if we condition on $\tf(x) = 1/8$ so that $x_2 = -1$, we have 
\[ p^*(x) = \fr{2} - \frac{3}{8} + \fr{8}x_1 = \frac{1 + x_1}{8}.\] 
Conditioned on  $\tf(x) = 7/8$ so that $x_2 =1$, we have 
\[ p^*(x) = \fr{2} + \frac{3}{8} - \fr{8}x_1 = \frac{7 - x_1}{8}.\]
Hence we have
\[ \E_{v \sim \fD}\lt[\lt|\min_{h \in \Lin(\mC)}\E[(\y - h(\x))^2|f(\x) =v]\rt|\rt] = \E[(y - p^*(x))^2] = \Var[\y], \]
whereas the variance decomposition of squared loss gives
\begin{align*}
    \E[(\y - \tf(\x))^2] &=  \E[(\y - p^*(\x))^2] +  \E[(p^*(\x) - \tf(\x))^2] \\
    &= \Var[\y] + \fr{64}\E[(x_1x_2)^2]\\
    &= \Var[\y] + \fr{64}.
\end{align*}
Hence $\tf$ is not a $(\ell_2, \Lin(\mC), \delta)$-swap agnostic learner for $\delta < 1/64.$

To see that $f$ is not multicalibrated for small $\alpha$, observe that conditioned on $x_2 \in \pmo$, the correlation between $x_1$ and $\tp - p^*$ is $1/8$. 
\end{proof}

Note that item (1) above separates swap omniprediction from omniprediction and agnostic learning.
This separation can also be derived from \cite[Theorem 7.5]{omni} which separated (standard) omniprediction from agnostic learning, since swap omniprediction implies standard omniprediction.

\paragraph{Comparing notions for GLM losses.}
When we restrict our attention to $\lglm$, in fact, the notions of swap loss OI and swap omniprediction are equivalent.
The key observation here is that $\partial \lglm \circ \mC = \mC$, as shown in \cite{lossoi}.
Paired with Theorem~\ref{thm:main} and Theorem~\ref{thm:swap-eq}, we obtain the following collapse.

\begin{claim}
The notions of $(\lglm, \mC, \alpha_1)$-swap loss OI and $(\lglm, \mC, \alpha_2)$-swap omniprediction are equivalent.
\end{claim}
\begin{proof}
To see this, note that by Theorem~\ref{thm:swap-eq}, $(\lglm,\mC,\alpha_1)$-swap loss OI is equivalent to $(\partial \lglm \circ \mC,\alpha_1')$-swap multicalibration.
We know from Theorem~\ref{thm:main} that this is also equivalent to $(\partial\lglm \circ \mC,\alpha_2)$-swap omniprediction.
So, by the fact that $\partial \lglm \circ \mC = \mC$, we have the claimed equivalence.
\end{proof}

Finally, we know that loss OI implies omniprediction for $\lglm$, since this holds true for all $\mL$.
We do not know if these notions are equivalent for $\lglm$, since the construction in Lemma \ref{lem:sep1} used the $\ell_4$ loss which does not belong to $\lglm$.

\clearpage
\paragraph{Acknowledgements.}
We thank Sumegha Garg, Christopher Jung,  Salil Vadhan, and Udi Wieder for their helpful comments on this work.

\bibliographystyle{alpha}
\bibliography{refs}

\newcommand{\etalchar}[1]{$^{#1}$}
\begin{thebibliography}{GHHK{\etalchar{+}}23}

\bibitem[BL20]{BlumLyk20}
Avrim Blum and Thodoris Lykouris.
\newblock {Advancing Subgroup Fairness via Sleeping Experts}.
\newblock In {\em 11th Innovations in Theoretical Computer Science Conference
  (ITCS 2020)}, volume 151, pages 55:1--55:24, 2020.

\bibitem[BLM01]{SBD2}
Shai Ben{-}David, Philip~M. Long, and Yishay Mansour.
\newblock Agnostic boosting.
\newblock In {\em 14th Annual Conference on Computational Learning Theory,
  {COLT}}, 2001.

\bibitem[BM07a]{BlumM07}
Avrim Blum and Yishay Mansour.
\newblock From external to internal regret.
\newblock {\em J. Mach. Learn. Res.}, 8:1307--1324, 2007.

\bibitem[BM07b]{BMchapter}
Avrim Blum and Yishay Mansour.
\newblock Learning, regret minimization, and equilibria.
\newblock In {\em Algorithmic Game Theory, Noam Nisan, Tim Roughgarden, Eva
  Tardos, and Vijay Vazirani editors}. Cambridge University Press, 2007.

\bibitem[CBL06]{CBL}
Nicolò Cesa-Bianchi and Gábor Lugosi.
\newblock {\em Prediction, learning, and games}.
\newblock Cambridge University Press, 2006.

\bibitem[DDZ23]{DengDZ23}
Zhun Deng, Cynthia Dwork, and Linjun Zhang.
\newblock Happymap : {A} generalized multicalibration method.
\newblock In {\em 14th Innovations in Theoretical Computer Science Conference,
  {ITCS} 2023}, volume 251 of {\em LIPIcs}, pages 41:1--41:23. Schloss Dagstuhl
  - Leibniz-Zentrum f{\"{u}}r Informatik, 2023.

\bibitem[DKR{\etalchar{+}}21]{oi}
Cynthia Dwork, Michael~P. Kim, Omer Reingold, Guy~N. Rothblum, and Gal Yona.
\newblock Outcome indistinguishability.
\newblock In {\em ACM Symposium on Theory of Computing (STOC'21)}, 2021.

\bibitem[DKR{\etalchar{+}}22]{dwork2022beyond}
Cynthia Dwork, Michael~P. Kim, Omer Reingold, Guy~N. Rothblum, and Gal Yona.
\newblock Beyond bernoulli: Generating random outcomes that cannot be
  distinguished from nature.
\newblock In {\em The 33rd International Conference on Algorithmic Learning
  Theory}, 2022.

\bibitem[DLLT23]{DworkLLT}
Cynthia Dwork, Daniel Lee, Huijia Lin, and Pranay Tankala.
\newblock New insights into multi-calibration, 2023.

\bibitem[FV98]{FosterV98}
Dean~P. Foster and Rakesh~V. Vohra.
\newblock Asymptotic calibration.
\newblock {\em Biometrika}, 85(2):379--390, 1998.

\bibitem[FV99]{FV99}
Dean Foster and Rakesh Vohra.
\newblock Regret in the on-line decision problem.
\newblock {\em Games and Economic Behavior}, 29:7--35, 1999.

\bibitem[GHHK{\etalchar{+}}23]{GHHKRS}
Ira Globus-Harris, Declan Harrison, Michael Kearns, Aaron Roth, and Jessica
  Sorrell.
\newblock Multicalibration as boosting for regression, 2023.

\bibitem[GHK{\etalchar{+}}23]{lossoi}
Parikshit Gopalan, Lunjia Hu, Michael~P. Kim, Omer Reingold, and Udi Wieder.
\newblock Loss minimization through the lens of outcome indistinguishability.
\newblock In {\em Innovations in theoretical computer science (ITCS'23)}, 2023.

\bibitem[GJKR18]{GillenJKR18}
Stephen Gillen, Christopher Jung, Michael~J. Kearns, and Aaron Roth.
\newblock Online learning with an unknown fairness metric.
\newblock In {\em Advances in Neural Information Processing Systems 31: Annual
  Conference on Neural Information Processing Systems 2018, NeurIPS 2018,
  December 3-8, 2018, Montr{\'{e}}al, Canada}, pages 2605--2614, 2018.

\bibitem[GJN{\etalchar{+}}22]{GuptaJNPR22}
Varun Gupta, Christopher Jung, Georgy Noarov, Mallesh~M. Pai, and Aaron Roth.
\newblock Online multivalid learning: Means, moments, and prediction intervals.
\newblock In {\em 13th Innovations in Theoretical Computer Science Conference,
  {ITCS} 2022}, volume 215 of {\em LIPIcs}, pages 82:1--82:24. Schloss Dagstuhl
  - Leibniz-Zentrum f{\"{u}}r Informatik, 2022.

\bibitem[GKR{\etalchar{+}}22]{omni}
Parikshit Gopalan, Adam~Tauman Kalai, Omer Reingold, Vatsal Sharan, and Udi
  Wieder.
\newblock Omnipredictors.
\newblock In {\em Innovations in Theoretical Computer Science (ITCS'2022)},
  2022.

\bibitem[GKR23]{swap-neurips}
Parikshit Gopalan, Michael~P. Kim, and Omer Reingold.
\newblock Swap agnostic learning, or characterizing omniprediction via
  multicalibration.
\newblock In {\em NeurIPS}, 2023.

\bibitem[GKSZ22]{GopalanKSZ22}
Parikshit Gopalan, Michael~P. Kim, Mihir Singhal, and Shengjia Zhao.
\newblock Low-degree multicalibration.
\newblock In {\em Conference on Learning Theory, 2-5 July 2022, London, {UK}},
  volume 178 of {\em Proceedings of Machine Learning Research}, pages
  3193--3234. {PMLR}, 2022.

\bibitem[GRSW22]{gopalan2021multicalibrated}
Parikshit Gopalan, Omer Reingold, Vatsal Sharan, and Udi Wieder.
\newblock Multicalibrated partitions for importance weights.
\newblock In {\em International Conference on Algorithmic Learning Theory, 29-1
  April 2022, Paris, France}, volume 167 of {\em Proceedings of Machine
  Learning Research}, pages 408--435. {PMLR}, 2022.

\bibitem[Hau92]{haussler1992decision}
David Haussler.
\newblock Decision theoretic generalizations of the pac model for neural net
  and other learning applications.
\newblock {\em Information and Computation}, 100(1):78--150, 1992.

\bibitem[HKRR18]{hkrr2018}
{\'{U}}rsula H{\'{e}}bert{-}Johnson, Michael~P. Kim, Omer Reingold, and Guy~N.
  Rothblum.
\newblock Multicalibration: Calibration for the (computationally-identifiable)
  masses.
\newblock In {\em Proceedings of the 35th International Conference on Machine
  Learning, {ICML}}, 2018.

\bibitem[JLP{\etalchar{+}}21]{JungLPRV21}
Christopher Jung, Changhwa Lee, Mallesh~M. Pai, Aaron Roth, and Rakesh Vohra.
\newblock Moment multicalibration for uncertainty estimation.
\newblock In {\em Conference on Learning Theory, {COLT} 2021, 15-19 August
  2021, Boulder, Colorado, {USA}}, volume 134 of {\em Proceedings of Machine
  Learning Research}, pages 2634--2678. {PMLR}, 2021.

\bibitem[KGZ19]{kgz}
Michael~P. Kim, Amirata Ghorbani, and James Zou.
\newblock Multiaccuracy: Black-box post-processing for fairness in
  classification.
\newblock In {\em Proceedings of the 2019 AAAI/ACM Conference on AI, Ethics,
  and Society}, pages 247--254, 2019.

\bibitem[KKG{\etalchar{+}}22]{kim2022universal}
Michael~P Kim, Christoph Kern, Shafi Goldwasser, Frauke Kreuter, and Omer
  Reingold.
\newblock Universal adaptability: Target-independent inference that competes
  with propensity scoring.
\newblock {\em Proceedings of the National Academy of Sciences}, 119(4), 2022.

\bibitem[KMR17]{KleinbergMR17}
Jon~M. Kleinberg, Sendhil Mullainathan, and Manish Raghavan.
\newblock Inherent trade-offs in the fair determination of risk scores.
\newblock In {\em 8th Innovations in Theoretical Computer Science Conference,
  {ITCS}}, 2017.

\bibitem[KNRW17]{KNRW17}
Michael Kearns, Seth Neel, Aaron Roth, and Zhiwei~Steven Wu.
\newblock Preventing fairness gerrymandering: Auditing and learning for
  subgroup fairness.
\newblock {\em arXiv preprint arXiv:1711.05144}, 2017.

\bibitem[KV94]{kearns1994introduction}
Michael~J Kearns and Umesh Vazirani.
\newblock {\em An introduction to computational learning theory}.
\newblock MIT press, 1994.

\bibitem[Pie12]{pietrzak2012cryptography}
Krzysztof Pietrzak.
\newblock Cryptography from learning parity with noise.
\newblock In {\em SOFSEM}, volume~12, pages 99--114. Springer, 2012.

\bibitem[Val84]{valiant1984theory}
Leslie~G Valiant.
\newblock A theory of the learnable.
\newblock {\em Communications of the ACM}, 27(11):1134--1142, 1984.

\end{thebibliography}

\end{document}